\documentclass[10pt,twocolumn,letterpaper]{article}

\usepackage{iccv}
\usepackage{times}
\usepackage{epsfig}
\usepackage{graphicx}
\usepackage{amsmath}
\usepackage{amssymb}

\usepackage{algorithm}
\usepackage{algorithmic}

\usepackage{caption}
\usepackage{subcaption}

\usepackage{amsthm} 

\newtheorem{theorem}{Theorem}
\newtheorem{lemma}{Lemma}
\newtheorem{proposition}{Proposition}

\newtheorem{definition}{Definition}

\newcommand{\argmin}{\operatornamewithlimits{argmin}}

\usepackage[pagebackref=true,breaklinks=true,letterpaper=true,colorlinks,bookmarks=false]{hyperref}

\iccvfinalcopy 

\def\iccvPaperID{352} 

\ificcvfinal
\begin{document}

\title{Parsimonious Labeling}

\author{Puneet K. Dokania\\
CentraleSupelec\\
INRIA Saclay\\
{\tt\small puneetkdokania@gmail.com}
\and
M. Pawan Kumar\\
CentraleSupelec\\
INRIA Saclay\\
{\tt\small pawan.kumar@ecp.fr}
}

\maketitle


%
%
%
%
%
%

\begin{abstract}
We propose a new family of discrete energy minimization problems, which we call parsimonious labeling. Specifically, our
energy functional consists of unary potentials and high-order clique potentials. While the unary potentials are
arbitrary, the clique potentials are proportional to the {\em diversity} of set of the unique labels assigned to
the clique. Intuitively, our energy functional encourages the labeling to be parsimonious, that is, use as few labels
as possible. This in turn allows us to capture useful cues for important computer vision applications such as stereo
correspondence and image denoising. Furthermore, we propose an efficient
graph-cuts based algorithm for the parsimonious labeling problem that provides strong theoretical guarantees on the quality of the solution.
Our algorithm consists of three steps. First, we approximate a given diversity using a mixture of a novel hierarchical $P^n$ Potts model.
Second , we use a divide-and-conquer approach for each mixture component, where each subproblem is solved using an effficient
$\alpha$-expansion algorithm. This provides us with a small number of putative labelings, one for each mixture component.
Third, we choose the best putative labeling in terms of the energy value.
Using both sythetic and standard real datasets, we show that our algorithm significantly outperforms other graph-cuts based approaches.
\end{abstract}

\section{Introduction}
The labeling problem provides an intuitive formulation for several problems in computer vision and related areas. Briefly, the labeling problem is defined using a
set of random variables, each of which can take a value from a finite and discrete label set. The assignment of values to all the variables is referred to as a
labeling. In order to quantatively distinguish between the large number of putative labelings, we are provided with an energy functional that maps a labeling to
a real number. The energy functional consists of two types of terms: (i) the unary potential, which depends on the label assigned to one random variable at a time; and
(ii) the clique potential, which depends on the labels assigned to a set of random variables.
The goal of the labeling problem is to obtain the labeling that minimizes the energy.

Perhaps a well-studied special case of the labeling problem is the metric labeling problem~\cite{Boykov01GraphCut,Kleinberg_Metric_1999}. Here, the unary potentials are arbitrary. However, the clique potentials are specified by a user-defined metric distance function of the label space. Specifically, the clique potentials satisfy the following two properties: (i) each clique potential depends on two random variables; and (ii) the value of the clique potential (also referred to as the pairwise potential) is proportional
to the metric distance between the labels assigned to the two random variables. Metric labeling has been used to formulate several problems in low-level computer
vision, where the random variables correspond to image pixels. In such scenarios, it is natural to encourage two random variables that correspond to two nearby pixels
in the image to take similar labels. However, by restricting the size of the cliques to two, metric labeling fails to capture more informative high-order cues. For
example, it cannot encourage an arbitrary sized set of similar pixels (such as pixels that define a homogeneous superpixel) to take similar labels. 

We propose a natural generalization of the metric labeling problem for high-order potentials, which we call parsimonious labeling. Similar to metric labeing, our
energy functional consists of arbitrary unary potentials. However, the clique potentials can be defined on any set of random variables, and their value depends on
the set of unique labels assigned to the random variables in the clique. In more detail, the clique potential is defined using the recently proposed notion of
a {\em diversity}~\cite{Bryant13Diversities}, which generalizes metric distance functions to all subsets of the label set. By minimizing the diversity, our energy functional encourages
the labeling to be parsimonious, that is, use as few labels as possible. This in turn allows us to capture useful cues for important low-level computer vision applications. 

In order to be practically useful, we require an computationally feasible solution for parsimonious labeling. To this end, we design a novel three step algorithm that uses 
an efficient graph cuts based method as its key ingredient. The first step of our algorithm approximates a given diversity as a mixture of a novel hierarchical $P^n$
Potts model (a generalization of the $P^n$ Potts model~\cite{Kohli07_P3Beyond_CVPR}). The second step of our algorithm solves the labeling problem corresponding to each component of the
mixture via a divide-and-conquer approach, where each subproblem is solved using $\alpha$-expansion~\cite{Veksler_PhDthesis_1999}. This provides us with a small set of putative labelings, each
corresponding to a mixture component. The third step of our algorithm simply chooses the putative labeling with the minimum energy. Using both sythetic and real datasets,
we show that our overall approach provides accurate results for various computer vision applications.


\section{Related Work}
In last few years the research community have witnessed many successful applications of high-order random fields to solve many low level vision related problems such as disparity estimation, image restoration, and object segmentation \cite{Delong_LabelCost_2010}\cite{Zehiry_Curvature_10}\cite{Fitzgibbon_Rendering_03} \cite{Kohli_RobustPn_2008}\cite{Ladicky_COOC_2010}\cite{Lan_BPHighOrder_06}\cite{Tarlow_HOP_10}\cite{Vicente_Joint_09}\cite{Woodford_Stereo_08}. In this work, our focus is on methods that (i) rely on efficient move-making algorithms based on graph cuts; (ii) provide a theoretical guarantee on the quality of the solution. Below, we discuss the work most closely related to ours in more detail.

Kohli et al. \cite{Kohli07_P3Beyond_CVPR} proposed the $P^n$ Potts model, which enforces label consistency over a set of random variables. In \cite{Kohli_RobustPn_2008}, they presented a robust version of the $P^n$ Potts model that takes into account the number of random variables that have been assigned an inconsistent label. Both the $P^n$ Potts model and its robust version lend themselves to the efficient $\alpha-$expansion algorithm \cite{Kohli07_P3Beyond_CVPR,Kohli_RobustPn_2008}. Furthermore, the $\alpha-$expansion algorithm also provides a multiplicative bound on the energy of the estimated labeling with respect to the optimal labeling. While the robust $P^n$ Potts model has been shown to be very useful for semantic segmentation, our generalization of the $P^n$ Potts model offers a natural extension of the metric labeling problem and is therefore more widely applicable to several low-level computer vision applications. Delong et al. \cite{Delong_LabelCost_2010} propose a global clique potential that is based on the cost of using a label or a subset of labels in the labeling of the random variables. Similar to the $P^n$ Potts model, the label cost based potential can also be minimized using $\alpha-$expansion. However, the theoretical guarantee provided by $\alpha-$expansion is an additive bound, which is not invariant to reparameterization of the energy function. Delong et al. \cite{Delong_HierCost_2012} also proposed an extension of their work to hierarchical costs. However, the assumption of a given hierarchy over the label set limits its application in practice.

Independently, Ladicky {\em et al.} \cite{Ladicky_COOC_2010}  proposed a global co-occurrence cost based high order model for a much wider class of energies that encourage the use of a small set of labels in the estimated labeling. Theoretically, the only constraint that \cite{Ladicky_COOC_2010} enforces in high order clique potential is that it should be monotonic in the label set. In other words, the problem addressed in \cite{Ladicky_COOC_2010} can be regarded as a generalization of parsimonious labeling. However, they approximately optimize an upperbound on the actual energy functional which does not provide any optimality guarantees. In our experiments, we demonstrate that our move-making algorithm significantly outperforms their approach for the special case of parsimonious labeling.

\section{Preliminaries}
\paragraph{The labeling problem.} Consider a random field defined over a set of random variables ${\bf x} = \{x_1, \cdots, x_N\}$ arranged in a predefined lattice $\mathcal{V} = \{1, \cdots, N\}$. Each random variable can take a value from a discrete label set $\mathcal{L} = \{l_1, \cdots, l_H \}$. Furthermore, let $\mathcal{C}$ denote the set of maximal cliques. Each maximal clique consists of a set of random variables that are all connected to each other in the lattice. A labeling is defined as the assignment or mapping of random variables to the labels. To assess the quality of each labeling ${\bf x}$ we define an energy functional as:
\begin{eqnarray}
E({\bf x}) =   \sum_{i \in \mathcal{V} } \theta_i(x_i)  +\sum_{c \in \mathcal{C}} \theta_c({\bf x}_c)
\label{eq:mrfGeneral}
\end{eqnarray}

where $\theta_i(x_i)$ is any arbitrary unary potential of assigning a label $x_i$ to the random variable $i$, and $\theta_c({\bf x}_c)$ is a clique potential for assigning the labels ${\bf x}_c$ to the variables in the clique $c$. We assume that the clique potentials are non-negative. As will be seen shortly, this assumption is satisfied by the new family of energy functionals proposed in our paper. The total number of putative labelings is $H^N$, each of which can be assessed using its corresponding energy value. Within this setting, \emph{the labeling problem} is to find the labeling that corresponds to the minimum energy according to the functional (\ref{eq:mrfGeneral}). Formally, the labeling problem can be defined as: ${\bf x}^* = \argmin_{{\bf x}} E({\bf x})$. 

\paragraph{$P^n$ Potts model.} An important special case of the labeling problem, which will be used throughout this paper, is defined by the $P^n$ Potts model \cite{Kohli07_P3Beyond_CVPR}. The $P^n$ Potts model is a generalization of of the well known Potts model \cite{Potts_1952} for high-order energy functions (when cliques can be of arbitrary sizes). For a given clique, the $P^n$ Potts model is defined as:
\vspace{-2mm}
\begin{eqnarray}
\label{eq:pnPotts}
    \theta_c({\bf x}_c)= 
\begin{cases}
    \gamma^k,& \text{if } x_i = l_k, \forall i \in c\\
   \gamma^{max},              & \text{otherwise}
\end{cases}
\end{eqnarray}

where $\gamma^k$ is the cost of assigning all the nodes to label $l_k \in \mathcal{L}$, and $\gamma^{max} > \gamma_k, \forall l_k \in \mathcal{L}$. Intuitively, the $P^n$ Potts model enforces label consistency by assigning the cost of $\gamma^{max}$ if more than one labels are present in the given clique.

\paragraph{$\alpha$-expansion for $P^n$ Potts model.} In order to solve the labeling problem corresponding to the $P^n$ Potts model, Kohli {\em et al.} \cite{Kohli07_P3Beyond_CVPR} proposed to use the $\alpha-$expansion algorithm \cite{Veksler_PhDthesis_1999}. The $\alpha-$expansion algorithm starts with an initial labeling, for example, by assigning each random variable to the label $l_1$. At each iteration, the algorithm moves to a new labeling by searching over a large {\em move space}. Here, the move space is defined as the set of labelings where each random variable is either assigned its current label or the label $\alpha$. The key result that makes $\alpha-$expansion a computationally feasible algorithm for the $P^n$ Potts model is that the minimum energy labeling within a move-space can be obtained using a single minimum st-cut operation on a graph that consists of a small number (linear in the size of the variables and the cliques) of vertices and arcs. The algorithm terminates when the energy cannot be reduced further for any choice of the label $\alpha$. We refer the reader to \cite{Kohli07_P3Beyond_CVPR} for further details.

\paragraph{Multiplicative Bound.} The labeling problem, and many of its special cases including the one defined by the $P^n$ Potts model, is known to be NP-hard. However, due to its practical importance, many approximate algorithms have been proposed in the literature (for example, the aforementioned $\alpha-$expansion algorithm for the $P^n$ Potts model). An intuitive and commonly used measure of the accuracy of an approximation algorithm is the multiplicative bound. Formally, the multiplicative bound of a given algorithm is said to be $B$ if the following condition is satisfied for all possible values of unary potential $\theta_i(.)$, and clique potentials $\theta_c({\bf x}_c)$:
\begin{eqnarray}
\label{eq:multiplicativeBound}
\sum_{i \in \mathcal{V} } \theta_i(\hat{x}_i)  +\sum_{c \in \mathcal{C}} \theta_c(\hat{{\bf x}}_c) \leq \sum_{i \in \mathcal{V} } \theta_i(x^*_i)  + B \sum_{c \in \mathcal{C}} \theta_c({\bf x}^*_c)
\end{eqnarray}
Here, $\hat{{\bf x}}$ is the labeling estimated by the algorithm and  ${\bf x}^*$ is a globally optimal labeling. By definition of an optimal labeling (one that has the minimum energy), the multiplicative bound will always be greater than or equal to one \cite{Pawan14_Rounding_NIPS}. 

\paragraph{Multiplicative Bound for the $\alpha$-expansion algorithm for the $P^n$ Potts model.} Using the $\alpha-$expansion algorithm for the $P^n$ potts model we obtain the multiplicative bound of $\lambda \min(\mathcal{M}, |\mathcal{L}|)$, where, $\mathcal{M}$ is the size of the largest maximal clique in the graph, $ |\mathcal{L}|$ is the number of labels, and $\lambda$ is defined as below \cite{Gould_AlphabetSoup_09CVPR}: 

\begin{eqnarray}
\label{eq:pnPottsLambda}
   \gamma^{min} = \min_{k \in \mathcal{L}} \gamma_k, \quad \lambda= 
\begin{cases}
    \frac{\gamma^{max}}{\gamma^{min}},& \text{if } \gamma^{min} \neq 0\\
   \gamma^{max},              & \text{otherwise}
\end{cases}
\end{eqnarray}

\section{Parsimonious Labeling}
The parsimonious labeling problem is defined using an energy functional that consists of unary potentials and clique potentials defined over cliques of arbitrary sizes. While the parsimonious labeling problem places no restrictions on the unary potentials, the clique potentials are specified using a {\em diversity} function \cite{Bryant13Diversities}. Before describing the parsimonious labeling problem in detail, we briefly define the diversity function for the sake of completion.

\begin{definition}
A diversity is a pair $(\mathcal{L}, \delta)$, where $\mathcal{L}$ is the set of labels and $\delta$ is a non-negative function defined on finite subsets of $\mathcal{L}$, $\delta: \Gamma \rightarrow \mathbb{R}, \forall \Gamma \subseteq \mathcal{L}$, satisfying following properties:
\begin{itemize}

\item Non Negativity: $\delta(\Gamma) \geq 0$, and $\delta(\Gamma) = 0$, iff, $|\Gamma| \leq 1$.

\item Triangular Inequality: if $\Gamma_2 \neq \emptyset $, $\delta(\Gamma_1 \cup \Gamma_2 ) + \delta(\Gamma_2 \cup \Gamma_3 ) \geq \delta(\Gamma_1 \cup \Gamma_3 ), \forall \Gamma_1, \Gamma_2, \Gamma_3 \subseteq \mathcal{L}.$  

\item Monotonicity: $\Gamma_1 \subseteq \Gamma_2$ implies $\delta(\Gamma_1) \leq \delta(\Gamma_2)$

\end{itemize}
\end{definition}

Using a diversity function, we can define a clique potential as follows. We denote by $\Gamma({\bf x}_c)$ the set of unique labels in the labeling of the clique c. Then, $\theta_c({\bf x}_c) = w_c\delta(\Gamma({\bf x}_c))$, where $\delta$ is a diversity function and $w_c$ is the non-negative weight corresponding to the clique $c$. Formally, the {\em parsimonious labeling problem } amounts to minimizing the following energy functional:

\begin{eqnarray}
\label{eq:parcEnergy}
E({\bf x}) =   \sum_{i \in \mathcal{V} } \theta_i(x_i)  +\sum_{c \in \mathcal{C}} w_c\delta(\Gamma({\bf x}_c))
\end{eqnarray}

Therefore, given a clique ${\bf x}_c$ and the set of unique labels $\Gamma({\bf x}_c)$ assigned to the random variables in the clique, the clique potential function for the parsimonious labeling problem is defined using $\delta(\Gamma({\bf x}_c))$, where $\delta: \Gamma({\bf x}_c) \rightarrow \mathbb{R}$ is a diversity function. 
 
Intuitively, diversities enforces parsimony by choosing a solution with less number of unique labels from a set of equally likely solutions, which makes it highly interesting for the computer vision community. This is an essential property in many vision problems, for example, in case of image segmentation, we would like to see label consistency within superpixels in order to preserve discontinuity. Unlike the $P^n$ Potts model the diversity does not enforce the label consistency very rigidly. It gives monotonic rise to the cost based on the number of labels assigned to the given clique.

An important special case of the parsimonious labeling problem is the {\em metric labeling problem}, which has been extensively studied in computer vision~\cite{Boykov01GraphCut} and theoretical computer science~\cite{Kleinberg_Metric_1999}. In metric labeling, the maximal cliques are of size two (pairwise) and the clique potential function is a metric distance function defined over the labels. Recall that a distance function $d:\mathcal{L}\times\mathcal{L}\rightarrow\mathbb{R}$ is a metric if and only if: (i) $d(.,.) \geq 0$; (ii) $d(i,j)+d(j,k)\geq d(i,k), \forall i,j,k$; and (iii) $d(i,j) = 0$ if and only if $i=j$. Notice that, there is a direct link between the metric distance function and the diversities. The diversities can be seen as the metric distance function over the sets of arbitrary sizes. In another words, diversities are the generalization of the metric distance function and boil down to a metric distance function if the input set is restricted to the subsets with cardinality of at most two. Another way of understanding the connection between metrics and diversities is that {\em every diversity induces a metric}. In other words, consider $d(l_i,l_i) = \delta({l_i})$ and $d(l_i,l_j) = \delta(\{l_i,l_j \})$. Using the properties of diversities, it can be shown that $d(\cdot,\cdot)$ is a metric distance function. Hence, in case of energy functional defined over pairwise cliques, the \emph{parsimonious labeling} problem reduces to the \emph{metric labeling} problem.

In the remaining part of this section we talk about a specific type of diversity called {\em the diameter diversity}, show its relation with the well known $P^n$ Potts model, and propose a {\em hierarchical $P^n$ Potts model} based on the diameter diversity defined over a hierarchical clustering (defined shortly). However, note that our approach is applicable to any general parsimonious labeling problem.

\paragraph{Diameter diversity.} Among many known diversities (\cite{Bryant12Hyperconvexity}), in this work, we are primarily interested in the \emph{diameter diversity}. Let $(\mathcal{L}, \delta)$ be a diversity and  $(\mathcal{L}, d)$ be the induced metric of $(\mathcal{L}, \delta)$, where $d:\mathcal{L} \times \mathcal{L} \rightarrow \mathbb{R}$  and $d(l_i,l_j) = \delta(\{l_i,l_j\}), \forall l_i,l_j \in \mathcal{L}$, then for all $\Gamma \subseteq \mathcal{L}$, the diameter diversity is defined as:

\begin{eqnarray}
\label{eq:diaDiversity}
\delta^{dia} (\Gamma) = \max_{l_i,l_j \in \Gamma} d(l_i,l_j).
\end{eqnarray}

Clearly, given the induced metric function defined over a set of labels, diameter diversity over any subset of labels gives the measure of how dissimilar (or diverse) the labels are. More the dissimilarity, based on the induced metric function, higher is the diameter diversity. Therefore, using diameter diversity as clique potentials enforces the similar labels to be together. Thus, a special case of {\em parsimonious labeling} in which the clique potentials are of the form of diameter diversity can be defined as below:

\begin{eqnarray}
\label{eq:diaDivEnergy}
E({\bf x}) =   \sum_{i \in \mathcal{V} } \theta_i(x_i)  +\sum_{c \in \mathcal{C}} w_c\delta^{dia}(\Gamma({\bf x}_c))
\end{eqnarray}

Notice that the diameter diversity defined over uniform metric is nothing but the $P^n$ Potts model where $\gamma_i = 0$. In what follows we define a generalization of the $P^n$ Potts model, the {\em hierarchical $P^n$ Potts model}, which will play a key role in the rest of the paper. 

\begin{figure}[t]			
	\centering
	\setlength{\tabcolsep}{-2em}
	\includegraphics[scale=0.55]{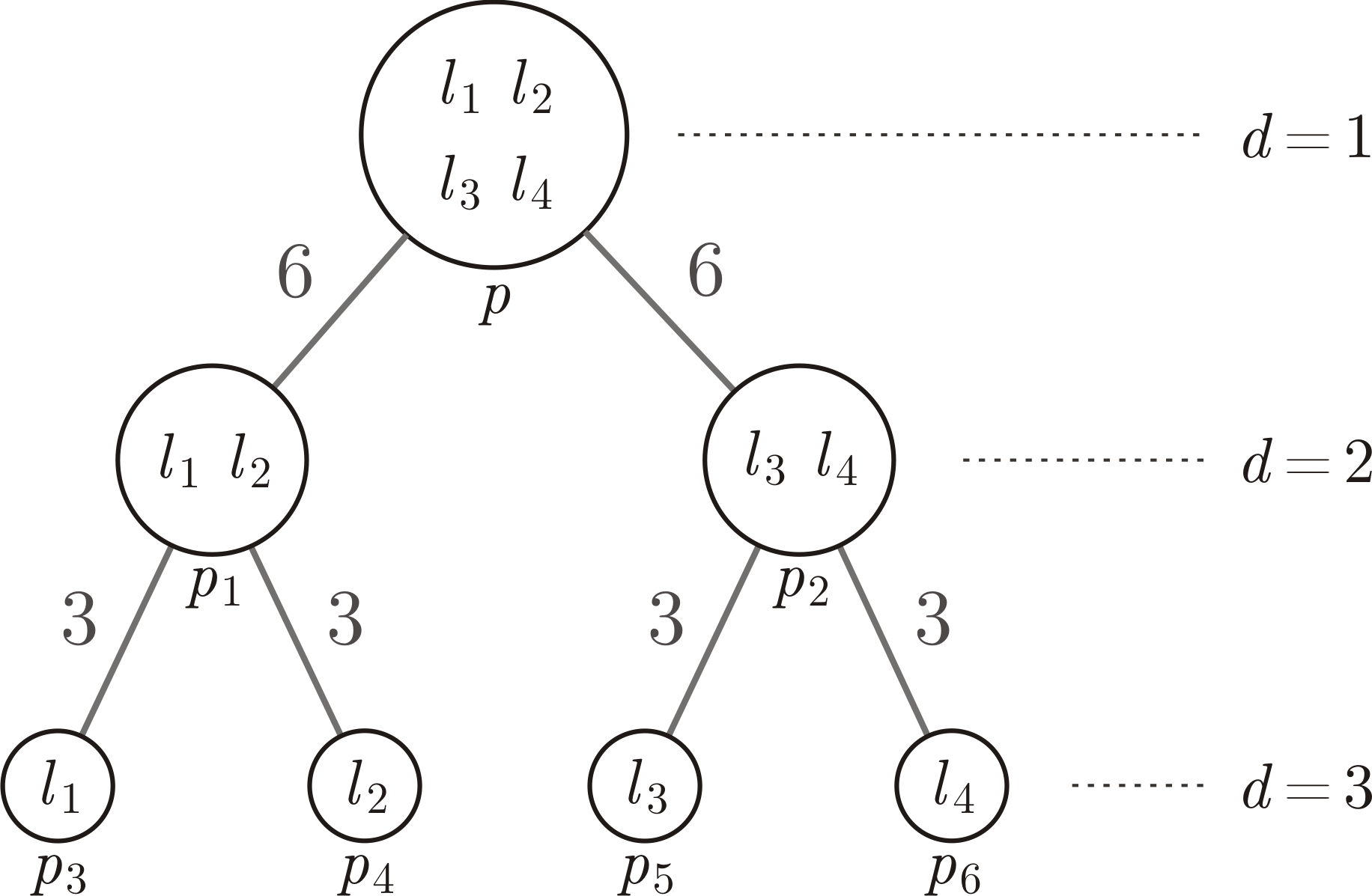}
	\caption{\label{fig:rHST}\emph{An example of r-{\sc hst} for $r=2$. The cluster associated with root $p$ contains all the labels. As we go down, the cluster splits into subclusters and finally we get the singletons, the leaf nodes (labels). The root is at depth of $d=1$ and leaf nodes at $d=3$. The metric defined over the r-{\sc hst} is denoted as $d^t(.,.)$, the shortest path between the inputs. For example, $d^t(l_1,l_3) = 18$ and $d^t(l_1,l_2) = 6$. {\em The diameter diversity} for the subset of labels at cluster $p$ is $\max_{ \{ l_i,l_j \} \in \{ l_1, l_2, l_3, l_4 \}} d^t(l_i, l_j) = 18.$}}
\end{figure}	

\paragraph{The Hierarchical $P^n$ Potts Model.} The hierarchical $P^n$ Potts model is a diameter diversity defined over a special type of metric known as the r-{\sc hst} metric. A rooted tree, as shown in figure (\ref{fig:rHST}), is said to be an r-{\sc hst}, or \emph{r-hierarchically well separated} \cite{Bartal98rhst} if it satisfy the following properties: (i) all the leaf nodes are the labels; (ii) all edge weights are positive; (iii) the edge lengths from any node to all of its children are the same; and (iv) on any root to leaf path the edge weight decrease by a factor of at least $r>1$. We can think of a r-{\sc hst} as a hierarchical clustering of the given label set $\mathcal{L}$. The root node represents the cluster at the top level of the hierarchy and contains all the labels. As we go down in the hierarchy, the clusters breaks down into smaller clusters until we get as many leaf nodes as the number of labels in the given label set. The metric distance function defined on this tree $d^t(.)$ is known as the r-{\sc hst} metric. In other words, the distance $d^t(\cdot,\cdot)$ between any two nodes in the given r-{\sc hst} is the shortest path distance between these nodes in the tree. The diameter diversity defined over $d^t(.,.)$ is called the hierarchical $P^n$ Potts model. The example of a diameter diversity defined over an r-{\sc hst} is given in the figure (\ref{fig:rHST}).

\section{The Hierarchical Move Making Algorithm}
In the first part of this section we propose a move making algorithm for the hierarchical $P^n$ Potts model (defined in the previous section). In the second part, we show how our hierarchical move making algorithm can be used to minimize the much more general \emph{parsimonious labeling} problem with optimality guarantees (tight multiplicative bound). 
\subsection{The Hierarchical Move Making Algorithm for the Hierarchical $P^n$ Potts Model}

\begin{algorithm}[tb]
\caption{The Move Making Algorithm for the Hierarchical $P^n$ Potts Model.}
\label{algo:hierPn}
\begin{algorithmic}[1]
\INPUT r-{\sc hst} Metric, $w_c, \forall c \in \mathcal{C}$, and $\theta_i(x_i), \forall i \in \mathcal{V}$
\STATE $d = D$, the leaf nodes
\REPEAT
\FOR {each $p \in \mathcal{N}(d)$}
\IF{$|\eta(p)| = 0$, leaf node}
\STATE $x_i^p = p, \forall i \in \mathcal{V}$
\ELSE
\STATE Fusion Move
\begin{eqnarray}
\label{eq:fusionMove}
\hat{\bf t}^p = \argmin_{{\bf t}^p \in \{1, \cdots, |\eta(p)| \}^N } E({\bf t}^p)
\end{eqnarray}
\STATE $x^p_i = x_i^{\eta(p,\hat{t}_i^p)}$.
\ENDIF
\ENDFOR

\STATE $d \leftarrow d-1$
\UNTIL{ $d \geq 1$}.
\end{algorithmic}
\end{algorithm}

In Hierarchical $P^n$ Potts model the clique potentials are of the form of the diameter diversity defined over a given r-{\sc hst} metric function. The move making algorithm proposed in this section to minimize such an energy functional is a divide-and-conquer based approach, inspired by the work of \cite{Pawan_SemiMetric_09UAI}. Instead of solving the actual problem, we divide the problem into smaller subproblems where each subproblem amounts to solving $\alpha-$expansion for the $P^n$ Potts model \cite{Kohli07_P3Beyond_CVPR}.  More precisely, given an r-{\sc hst}, each node of the r-{\sc hst} corresponds to a subproblem. We start with the bottom node of the r-{\sc hst}, which is a leaf node, and go up in the hierarchy solving each subproblem associated with the nodes encountered. 

In more detail, consider a node $p$ of the given r-{\sc hst}. Recall that any node $p$ in the r-{\sc hst} represents a cluster of labels denoted as $\mathcal{L}^p \subseteq \mathcal{L}$ (figure \ref{fig:rHST}). In another words, the leaf nodes of the subtree rooted at $p$ belongs to the $\mathcal{L}^p$. Thus, the subproblem defined at node $p$ is to find the labeling ${\bf x}^p$ where the label set is restricted to $\mathcal{L}^p$, as defined below. 

\begin{eqnarray}
\label{eq:energyNodeP}
{\bf x}^p = \argmin_{{\bf x} \in \mathcal{L}^p}  \Big( \sum_{i \in \mathcal{V} } \theta_i(x_i)  +\sum_{c \in \mathcal{C}} w_c\delta^{dia}(\Gamma({\bf x}_c)) \Big)
\end{eqnarray}

If $p$ is the root node, then the above problem (equation \ref{eq:energyNodeP}) is as difficult as the original labeling problem (since ${\cal L}^p = {\cal L}$). However, if $p$ is the leaf node then the solution of the problem associated with $p$ is trivial, $x_i^p = p$ for all $i \in \mathcal{V}$, which means, assign the label $p$ to all the random variables. This insight leads to the design of our approximation algorithm, where we start by solving the simple problems corresponding to the leaf nodes, and use the labelings obtained to address the more difficult problem further up the hierarchy. In what follows, we describe how the labeling of the problem associated with the node $p$, when $p$ is not the leaf node, is obtained using the labelings of its chidren node. 

\begin{figure}[t]			
	\centering
	\setlength{\tabcolsep}{-2em}
	\includegraphics[width=\linewidth]{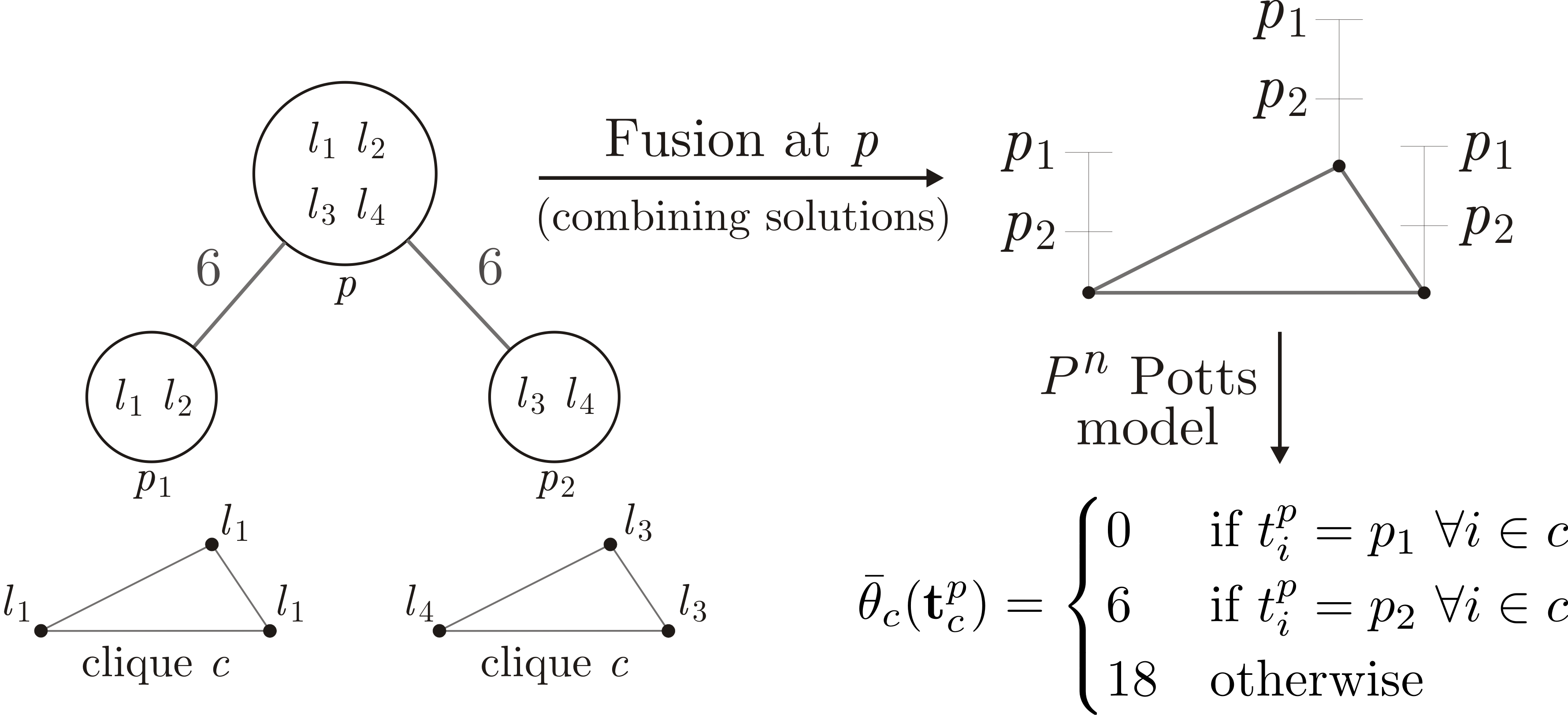}
	\caption{\label{fig:fusionPn}\emph{An example of solving the labeling problem at non-leaf node $(p)$ by combining the solutions of its child nodes $\{p_1, p_2\}$, given clique $c$ and the labelings that it has obtained at the child nodes. Note that the hierarchical clustering shown in this figure is the top two levels of the r-{\sc hst} shown in the figure (\ref{fig:rHST}), for a given clique $c$. The diameter diversity of the labeling of clique $c$ at node $p_1$ is 0 as it contains only one unique label $l_1$. The diameter diversity of the labeling at $p_2$ is $d^t(l_3,l_4) = 6$ and the label set at $p$ is 18. }}
\end{figure}	

\paragraph{Solving the Parent Labeling Problem}
Before delving into the details, let us define some notations for the purpose of clarity. Let $D$ be the depth (or the number of levels) in the given r-{\sc hst}. The root node being at the top level, depth of one. Let $\eta(p)$ denotes the set of child nodes associated with a non-leaf node $p$ and $\eta(p,k)$ denotes its $k^{th}$ child node. Recall that our approach is bottom up, therefore, for each child node of $p$ we already have a labeling associated with them. We denote the labeling associated with the $k^{th}$ child of the node $p$ as ${\bf x}^{\eta(p,k)}$. Thus, $x_i^{\eta(p,k)}$ denotes the label assigned to the $i^{th}$ random variable by the labeling of the $k^{th}$ child of the node $p$. We also define an $N$ dimensional vector ${\bf t}^p$, where each index of the the vector can take a value from the set denoting the child indices of node $p$, $\{1, \cdots, |\eta(p)| \}$, where $|\eta(p)|$ denotes the number of child nodes of $p$. More precisely, $t_i^p = k$ denotes that the label for the $i^{th}$ random variable comes from the $k^{th}$ child of the node $p$. Therefore, the labeling problem at node $p$ reduces to finding the optimal ${\bf t}^p$. Thus, the labeling problem at node $p$ amounts to finding the best child index $k \in \{1, \cdots, |\eta(p)| \}$ for each random variable $i \in \mathcal{V}$ so that the label assigned to the random variable comes from the labeling of the $k^{th}$ child.  

Using the above notations, associated with a ${\bf t}^p$ we define a new energy functional as:

\begin{eqnarray}
\label{eq:pnFusion}
E({\bf t}^p) =   \sum_{i \in \mathcal{V} } \bar{\theta}_i(t_i^p)  +\sum_{c \in \mathcal{C}} w_c \bar{\theta}_c({\bf t}_c^p)
\end{eqnarray}

where 

\begin{eqnarray}
\bar{\theta}_i(t_i^p) = \theta_i(x_i^{\eta(p,k)}) \quad \textit{if} \quad t_i^p = k
\end{eqnarray}

which says that the unary potential for $t_i^p = k$ is the unary potential associated to the $i^{th}$ random variable corresponding to the label $x_i^{\eta(p,k)}$. 

The new clique potential $\bar{\theta}_c({\bf t}_c^p)$ is as defined below:
\begin{eqnarray}
\label{eq:pnPottsFusion}
    \bar{\theta}_c({\bf t}_c^p)= 
\begin{cases}
    \gamma_k^p,& \text{if } t_i^p = k, \forall i \in c\\
   \gamma^p_{max},              & \text{otherwise}
\end{cases}
\end{eqnarray}
where $\gamma_k^p = \delta_{dia}( \Gamma ( {\bf x}_c^{\eta(p,k)} ) )$ is the diameter diversity of the set of unique labels associated with ${\bf x}_c^{\eta(p,k)}$ and $\gamma^p_{max} = \delta_{dia}(\mathcal{L}^p)$ is the diameter diversity of the set of labels associated with the cluster at node $p$. Recall that, because of the construction of the r-{\sc hst}, $\mathcal{L}^q \subset \mathcal{L}^p$ for all $q \in \eta(p)$. Hence, the monotonicity property of the diameter diversity ensures that $\gamma^p_{max} > \gamma_k^p, \forall k \in \eta(p)$. This is the sufficient criterion to prove that the potential function defined by equation (\ref{eq:pnPottsFusion}) is a $P^n$ Potts model. Therefore, the $\alpha-$expansion algorithm can be used to obtain the locally optimal ${\bf t}^p$ for the energy functional (\ref{eq:pnFusion}). Once we have obtained the locally optimal $\hat{\bf t}^p$, the labeling ${\bf x}^p$ at node $p$ can be trivially obtained as follows: $x^p_i = x_i^{\eta(p,\hat{t}_i^p)}$, which says that the final label of the $i^{th}$ random variable is the one assigned to it corresponding to the labeling of the $(\hat{t}_i^p)^{th}$ child of the node $p$. 

Figure (\ref{fig:fusionPn}) shows an instance of the above mentioned algorithm to combine the labelings of the child nodes to obtain the labeling of the parent node. The complete hierarchical move making algorithm for the hierarchical $P^n$ Potts model is shown in the Algorithm-\ref{algo:hierPn}.

 \paragraph{Multiplicative Bound.} Theorem-\ref{theorem:rhstMB} gives the multiplicative bound for the Move Making Algorithm for the Hierarchical $P^n$ Potts model.
 
\begin{theorem}
\label{theorem:rhstMB}
The move making algorithm for the hierarchical $P^n$ Potts model, Algorithm-\ref{algo:hierPn}, gives the multiplicative bound of $\left( \frac{r}{r-1} \right)   \min(\mathcal{M},|\mathcal{L}|)$ with respect to the global minima. Here, $\mathcal{M}$ is the size of the largest maximal-clique and $|\mathcal{L}|$ is the number of labels.

Proof: Given in Appendix.
\end{theorem}

\begin{algorithm}[tb]
\caption{The Move Making Algorithm for the Parsimonious Labeling Problem.}
\label{algo:parsimonyMoveMaking}
\begin{algorithmic}[1]
\INPUT Diversity $(\mathcal{L}, \delta)$; $w_c, \forall c \in \mathcal{C}$; $\theta_i(x_i), \forall i \in \mathcal{V}$; $\mathcal{L}$; $k$
\STATE Approximate the given diversity as the mixture of $k$ hierarchical $P^n$ Potts model using Algorithm-\ref{algo:divToHierPn}.
\FOR {each hierarchical $P^n$ Potts model in the mixture}
\STATE Use the hierarchical move making algorithm defined in the Algorithm-\ref{algo:hierPn}.
\STATE Compute energy corresponding to the solution obtained.
\ENDFOR
\STATE Choose the solution with the minimum energy.
\end{algorithmic}
\end{algorithm}

\begin{algorithm}[tb]
\caption{Diversity to Mixture of Hierarchical $P^n$ Potts model.}
\label{algo:divToHierPn}
\begin{algorithmic}[1]
\INPUT Diversity $(\mathcal{L}, \delta)$, $k$
\STATE Compute the induced metric, $d(.)$, where $d(l_i, l_j) = \delta(\{l_i, l_j \}), \forall l_i, l_j \in \mathcal{L}$.
\STATE Approximate $d(.)$ into mixture of $k$ r-{\sc hst} metrics $d^t(.)$ using the algorithm proposed in \cite{Fakcharoenphol04TightBound}.
\FOR {each r-{\sc hst} metrics $d^t(.)$}
\STATE Obtain the corresponding Hierarchical $P^n$ Potts model by defining the diameter diversity over $d^t(.)$
\ENDFOR
\end{algorithmic}
\end{algorithm}


\subsection{The Move Making Algorithm for the Parsimonious Labeling}
In the previous subsection, we proposed a hierarchical move making algorithm for the hierarchical $P^n$ Potts model. This restricted us to a very limited class of clique potentials. In this section we generalize our approach to the much more general {\em parsimonious labeling problem}.

The move making algorithm for the parsimonious labeling problem is shown in the Algorithm-(\ref{algo:parsimonyMoveMaking}). Given a diversity based clique potentials, clique weights, and the unary potentials, the  Algorithm-(\ref{algo:parsimonyMoveMaking}) approximates the diversity into a {\em mixture of hierarchical $P^n$ Potts models} and then use the previously defined hierarchical move making algorithm on each of the hierarchical $P^n$ Potts models. 

The algorithm for approximating a given diversity into a mixture of hierarchical $P^n$ Potts models is shown in Algorithm-(\ref{algo:divToHierPn}). The first and the third steps of the Algorithm-(\ref{algo:divToHierPn}) have already been discussed in the previous sections. The second step, which amounts to finding the mixture of r-{\sc hst} metrics for a given metric, can be solved using the randomized algorithm proposed in \cite{Fakcharoenphol04TightBound}. We refer the reader to \cite{Fakcharoenphol04TightBound} for further details of the algorithm for approximating a metric using a mixture of r-HST metrics.

\paragraph{Multiplicative Bound} Therorem-\ref{theo:diversityMB} gives the multiplicative bound for the {\em parsimonious labeling} labeling problem, when the clique potentials are any general diversity. 

\begin{theorem}
\label{theo:diversityMB}
The move making algorithm defined in Algorithm-\ref{algo:parsimonyMoveMaking} gives the multiplicative bound of $\left( \frac{r}{r-1} \right) (|\mathcal{L}|-1) (\log |\mathcal{L}|) \min(\mathcal{M},|\mathcal{L}|)$ for the {\em parsimonious labeling} problem (equation \ref{eq:parcEnergy}). Here, $\mathcal{M}$ is the size of the largest maximal-clique and $|\mathcal{L}|$ is the number of labels.

Proof: Given in the Appendix.
\end{theorem}

\section{Experiments}
We demonstrate the utility of the {\em parsimonious labeling} on both synthetic and real data. In case of synthetic data, we perform significant number of random experiments on big grid lattices and evaluate our method based on the energy and the time taken. To evaluate the modeling capabilities of the {\em parsimonious labeling}, we used it on two challenging real problems: (i) stereo matching, and (ii) image inpainting.  We use co-occurrence statistics based energy functional proposed by Ladicky {\em et al.}~\cite{Ladicky_COOC_2010} as our baseline. Theoretically, the only constraint that~\cite{Ladicky_COOC_2010} enforces on the clique potentials is that they must be monotonic in the label set. Therefore, can be regarded as the generalization of the parsimonious labeling. However, based on the synthetic and the real data results, supported by the theoretical guarantees, we show that the parsimonious labeling and the move making algorithm proposed in this work outperforms the more general work proposed in ~\cite{Ladicky_COOC_2010}. 

Recall that the energy functional of the parsimonious labeling problem is defined as:
 \begin{eqnarray}
\label{eq:parcEnergyAgain}
E({\bf x}) =   \sum_{i \in \mathcal{V} } \theta_i(x_i)  +\sum_{c \in \mathcal{C}} w_c\delta(\Gamma({\bf x}_c))
\end{eqnarray}

In our experiments, we frequently use the truncated linear metric. We define it below for the sake of completeness.

\begin{eqnarray}
\label{eq:truncatedLinear}
\theta_{i,j} (l_a, l_b) = \lambda \min(|l_a - l_b|, M), \forall l_a, l_b \in \mathcal{L}.
\end{eqnarray}
where $\lambda$ is the weight associated with the metric and $M$ is the truncation constant.

\subsection{Synthetic Data}
We consider following two cases: (i) when the hierarchical $P^n$ Potts model is given, and (ii) when a general diversity is given. In each of the two cases, we generate lattices of size $100 \times 100$, 20 labels, and use $\lambda=1$. The cliques are generated using a window of size $10 \times 10$ in a sliding window fashion. The unary potentials were randomly sampled from the uniform distribution defined over the interval $[0, 100]$. In the first case, we randomly generated $100$ lattices and random r-{\sc hst} trees associated with each lattice, ensuring that they satisfy the properties of the r-{\sc hst}. Each r-{\sc hst} was then converted into hierarchical $P^n$ Potts model by taking diameter diversity over each of them. This hierarchical $P^n$ Potts model was then used as the actual clique potential. We performed $100$ such experiments. On the other hand, in the second case, for a given value of the truncation $M$, we generated a truncated linear metric and $100$ lattices. We treated this metric as the induced metric of a diameter diversity and generated mixture of hierarchical $P^n$ Potts model using Algorithm-\ref{algo:divToHierPn}. Applied Algorithm-\ref{algo:hierPn} for the energy minimization over each hierarchical $P^n$ Potts model in the mixture and chose the one with the minimum energy. Notice that, in this case, the actual potential is the given diversity, not the generated hierarchical $P^n$ Potts models. Thus, the co-occurrence~\cite{Ladicky_COOC_2010} was given the actual diversity as the clique potentials. The method was evaluated using the given diversity as the clique potentials. We used four different values of the truncation factor $M \in \{ 1, 5, 10, 20\}$. For both the experiments, we used $7$ different values of $w_c$: $w_c \in \{0, 1, 2, 3, 4, 5, 100\}$.

The average energy and the time taken for both the methods and both the cases are shown in the figure (\ref{fig:synthetic}). It is evident from the figures that our method outperforms co-occurrence~\cite{Ladicky_COOC_2010} in both the cases, in term of time and the energy. In case the hierarchical $P^n$ Potts model is given, case (i), our method performs much better than co-occurrence~\cite{Ladicky_COOC_2010} because of the fact that it is directly minimizing the given potential. In case (ii), despite the fact that our method first approximates the given diversity into mixture of hierarchical $P^n$ Potts, it outperforms co-occurrence~\cite{Ladicky_COOC_2010}. This can be best supported by the fact that our algorithm has very tight multiplicative bound.

\begin{figure*}[ht!]
        \centering
        \begin{subfigure}[b]{0.2\linewidth}
                \includegraphics[width=\linewidth]{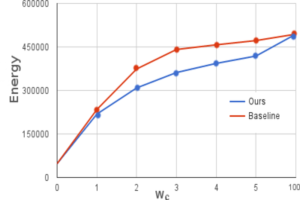}
                \caption{Energy}      
        \end{subfigure}%
        ~
        \begin{subfigure}[b]{0.2\linewidth}
                \includegraphics[width=\linewidth]{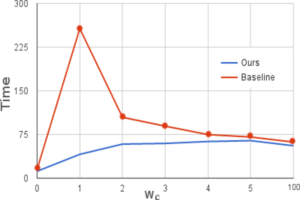}
                \caption{Time (in seconds)}      
        \end{subfigure}%
        ~
        \begin{subfigure}[b]{0.2\linewidth}
                \includegraphics[width=\linewidth]{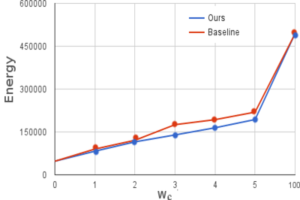}
                \caption{Energy}      
        \end{subfigure}%
        ~
        \begin{subfigure}[b]{0.2\linewidth}
                \includegraphics[width=\linewidth]{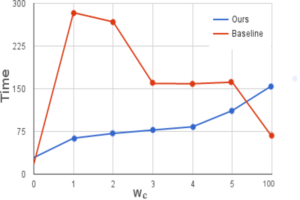}
                \caption{Time (in seconds)}      
        \end{subfigure}%
\caption{\label{fig:synthetic}\emph{Synthetic (Blue: Our, Red: Co-occ~\cite{Ladicky_COOC_2010}). The x-axis for all the figures is the weight associated with the cliques $(w_c)$. Figures (a) and (b) are the plots for the energy and the time when the hierarchical $P^n$ Potts model was assumed to be known. Figures (c) and (d) are the energy and the time plots for the case when a diversity (diameter diversity over truncated linear metric) was given as the clique potentials. Notice that in both the cases our method outperforms the baseline~\cite{Ladicky_COOC_2010} both in terms of energy and time. Also, for very high value of $w_c=100$, both the methods converges to the same labeling. This is expected as a very high value of $w_c$ enforces rigid smoothness by assigning everything to the same label.}}
\end{figure*}

\begin{figure*}[ht!]
        \centering
        \begin{subfigure}[b]{0.15\linewidth}
                \includegraphics[width=\linewidth]{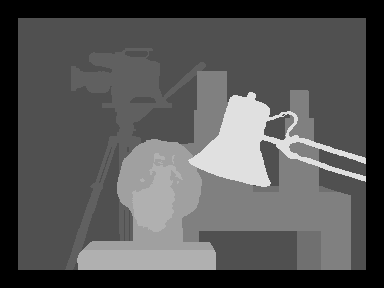}
                \caption{{{Tsukuba\\(Energy, Time)}}}      
        \end{subfigure}%
        ~
        \begin{subfigure}[b]{0.15\linewidth}
                \includegraphics[width=\linewidth]{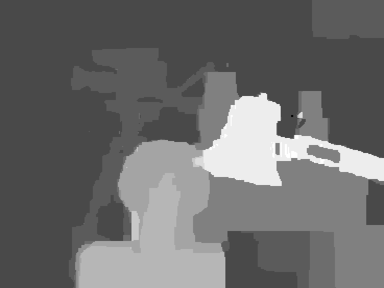}
                \caption{{Our\\ $({\bf 1195800}, 167)$}}      
        \end{subfigure}%
        ~
        \begin{subfigure}[b]{0.15\linewidth}
                \includegraphics[width=\linewidth]{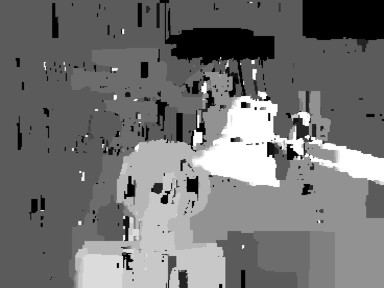}
                \caption{{Co-occ~\cite{Ladicky_COOC_2010}\\$(2202500, {\bf 95})$}}      
        \end{subfigure}%
        ~
        \begin{subfigure}[b]{0.15\linewidth}
                \includegraphics[width=\linewidth]{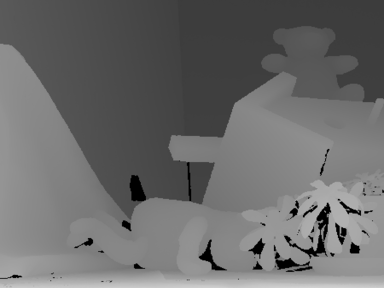}
                \caption{{ Teddy \\ (Energy, Time)}}      
        \end{subfigure}%
        ~
        \begin{subfigure}[b]{0.15\linewidth}
                \includegraphics[width=\linewidth]{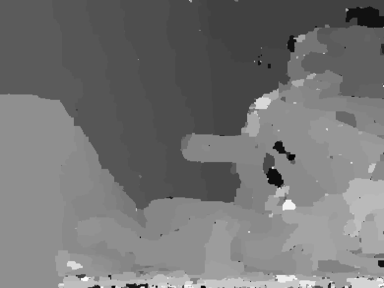}
                \caption{{Our \\$({\bf 1511206}, {\bf 287})$}}      
        \end{subfigure}%
        ~
        \begin{subfigure}[b]{0.15\linewidth}
                \includegraphics[width=\linewidth]{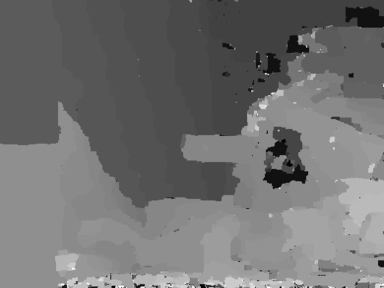}
                \caption{{Co-occ~\cite{Ladicky_COOC_2010}\\$(1519500,605)$}}
        \end{subfigure}%
\caption{\label{fig:stereoExp}\emph{Stereo Matching Results. Figures (a) and (d) are the ground truth disparity for the `tsukuba' and `teddy' respectively. Notice that our method outperforms the baseline Co-ooc~\cite{Ladicky_COOC_2010} in both the cases in terms of energy. From figure (b) and (e), we can clearly see the effect of `parsimonious labeling' as the regions are smooth and the discontinuity is preserved. }}
\end{figure*}

\begin{figure*}[ht!]
        \centering
        \begin{subfigure}[b]{0.15\linewidth}
                \includegraphics[width=\textwidth]{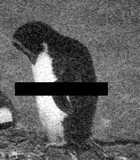}
                \caption{Penguin \\ (Energy, Time)}      
        \end{subfigure}%
        ~
        \begin{subfigure}[b]{0.15\linewidth}
                \includegraphics[width=\textwidth]{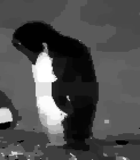}
                \caption{Our \\$({\bf 12516336},156)$}
        \end{subfigure}%
        ~
        \begin{subfigure}[b]{0.15\linewidth}
                \includegraphics[width=\textwidth]{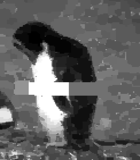}
                \caption{Co-oc~\cite{Ladicky_COOC_2010} \\$(14711806, {\bf 110})$}
        \end{subfigure}%
        ~
        \begin{subfigure}[b]{0.15\linewidth}
                \includegraphics[width=\textwidth]{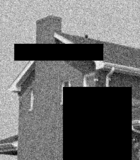}
                \caption{House \\ (Energy, Time)}
        \end{subfigure}%
        ~
        \begin{subfigure}[b]{0.15\linewidth}
                \includegraphics[width=\textwidth]{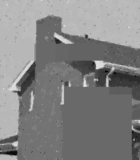}
                \caption{Our \\$({\bf 32799162},1014)$}
        \end{subfigure}%
        ~
        \begin{subfigure}[b]{0.15\linewidth}
                \includegraphics[width=\textwidth]{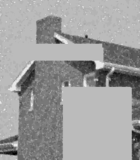}
                \caption{Co-oc~\cite{Ladicky_COOC_2010} \\$(38597848,{\bf 367})$}
        \end{subfigure}%
        \caption{\label{fig:inpaintExp}\emph{Image inpainting results. Figures (a) and (d) are the input images of `penguin' and `house' with added noise and obscured regions. Our method, (b) and (e), outperforms the baseline Co-ooc~\cite{Ladicky_COOC_2010} in both the cases in terms of energy. Figure (b) clearly shows the effect of `parsimonious labeling' as the regions are smooth and the discontinuity is preserved.}}
\end{figure*}

\subsection{Real Data}

In case of real data, the high-order cliques we used are the superpixels obtained using the mean-shift method \cite{MeanShift_PAMI_02}. The clique potentials used for the experiments are the diameter diversity of the {\em truncated linear} metric (equation (\ref{eq:truncatedLinear})). A truncated linear metric enforces smoothness in the pairwise setting, therefore, the diameter diversity of the truncated linear metric will naturally enforce smoothness in the high-order cliques, which is a desired cue for the two applications we are dealing with. In both the real experiments we used the following form of $w_c$ (for the high order cliques): $w_c = \exp^{-\frac{\rho({\bf x}_c)}{\sigma^2}}$, where $\rho({\bf x}_c)$ is the variance of the intensities of the pixels in the clique ${\bf x}_c$ and $\sigma$ is a hyperparameter.

\subsubsection{Stereo Matching} 

Given two rectified stereo pair of images, the problem of stereo matching is to find the disparity (gives the notion of depth) of each pixel in the reference image \cite{Tappen_Stereo_03,Szeliski_CompStudy_08}. In this work, we extended the standard setting of the stereo matching \cite{Szeliski_CompStudy_08} to high-order cliques and tested our method to the images, `tsukuba' and `teddy', from the widely used Middlebury stereo data set~\cite{Middlebury_Stereo_02}. The unaries were computed as the $L1-$norm of the difference in the {\sc rgb} values of the left and the right image pixels. Notice that the index for the right image pixel is the index for the left image pixel minus the disparity, which is the label. In case of `teddy' the unaries were trucated at 16. The weights $w_c$ for the pairwise cliques are set to be proportional to the $L1-$norm of the gradient $\Delta$ of the intensities of the neighbouring pixels. In case of `tsukuba', if $\Delta < 8 $, $w_c = 2$, otherwise $w_c = 1$. In case of `teddy', if $\Delta < 10 $, $w_c = 3$, otherwise $w_c = 1$.  As mentioned earlier, $w_c$ for the high-order cliques is set to be proportional to the variance. We used different values of $\sigma$, $\lambda$, and the truncation $M$. Because of the space constraints we are showing results for the following setting: for `tsukuba', $\lambda = 20$, $\sigma=100$ and $M=10$; for `teddy', $\lambda = 10$, $\sigma=1000$ and $M=1$. Figure (\ref{fig:stereoExp}) shows the results obtained. Notice that our method significantly outperforms the co-occurrence~\cite{Ladicky_COOC_2010} based method in terms of energy for both, `tsukuba' and `teddy'. {\em We show similar promising results for different parameters in the Appendix}.

\subsubsection{Image Inpainting and Denoising} 

Given an image with added noise and obscured regions (regions with missing pixels), the problem is to denoise the image and fill the obscured regions such that it is consistent with the surroundings. We performed this experiment on the images, `penguin' and `house', from the widely used Middlebury data set. The images under consideration are gray scale, therefore, there are 256 labels in the interval $[0,255]$, each representing an intensity value. The unaries for each pixel (or node) corresponding to a particular label, is the squared difference between the label and the intensity value at that pixel. The weights $w_c$ for the pairwise cliques are all set to one. For the high-order cliques, as mentioned earlier, $w_c$ are chosen to be proportional to the variance of the intensity of the participating pixels. We used different values of $\sigma$, $\lambda$, and the truncation $M$. Because of the space constraints we are showing results for the following setting: `penguin', the $\lambda = 40$, $\sigma=10000$ and $M=40$; for `house', the $\lambda = 30$, $\sigma=10$ and $M=40$. Figure~\ref{fig:inpaintExp} shows the results obtained. Notice that our method significantly outperforms the co-occurrence based method~\cite{Ladicky_COOC_2010} in terms of energy for both, `penguin' and `house'. {\em We show similar promising results for different parameters in the Appendix}.


\section{Discussion}

We proposed a new family of discrete optimization {\em parsimonious labeling}, a novel hierarchical $P^n$ Potts model, and move making algorithms to minimize energy functional for them. We gave very tight multiplicative bounds for the move making algorithms, applicable to all the `diversities'. An interesting direction for future research would be to explore different `diversities' and propose algorithms specific to them with better bounds. Another interesting future work would be to directly approximate `diversities' into mixture of hierarchical $P^n$ Potts model, without using the intermediate r-{\sc hst}.

\newpage
\onecolumn
\appendix    

\section{Additional Real Data Experiments and Analysis}
Recall that the energy functional of the parsimonious labeling problem is defined as:
 \begin{eqnarray}
E({\bf x}) =   \sum_{i \in \mathcal{V} } \theta_i(x_i)  +\sum_{c \in \mathcal{C}} w_c\delta(\Gamma({\bf x}_c))
\end{eqnarray}
where $\delta()$ is the diversity function defined over the set of unique labels present in the clique ${\bf x}_c$. In our experiments, we frequently use the truncated linear metric. We define it below for the sake of completeness.
\begin{eqnarray}
\label{eq:truncatedLinear1}
\theta_{i,j} (l_a, l_b) = \lambda \min(|l_a - l_b|, M), \forall l_a, l_b \in \mathcal{L}.
\end{eqnarray}
where $\lambda$ is the weight associated with the metric and $M$ is the truncation constant.

In case of real data, the high-order cliques are defined over the superpixels obtained using the mean-shift method \cite{MeanShift_PAMI_02}. The clique potentials used for the experiments are the diameter diversity of the {\em truncated linear} metric. A truncated linear metric (equation (\ref{eq:truncatedLinear1})) enforces smoothness in the pairwise setting, therefore, the diameter diversity of the truncated linear metric will naturally enforce smoothness in the high-order cliques, which is a desired cue for the two applications we are dealing with. 

In all the real experiments we use the following form of $w_c$ (for the high order cliques): $w_c = \exp^{-\frac{\rho({\bf x}_c)}{\sigma^2}}$, where $\rho({\bf x}_c)$ is the variance of the intensities of the pixels in the clique ${\bf x}_c$ and $\sigma$ is a hyperparameter.

In order to show the modeling capabilities of the {\em parsimonious labeling} we compare our results with the well known $\alpha-$expansion \cite{Veksler_PhDthesis_1999}, {\sc trws} \cite{Kolmogorov_TRWS_06PAMI}, and the Co-occ \cite{Ladicky_COOC_2010}. We also show the effect of clique sizes, which in our case are the superpixels obtained using the mean-shift algorithm, and the parameter $w_c$ associated with the cliques, for the purpose of understanding the behaviour of the {\em parsimonious labeling}.

\subsection{Stereo Matching}
Please refer to the paper for the description of the stereo matching problem. Figures (\ref{fig:stereoTeddyComparison}) and (\ref{fig:stereoTsukubaComparison}) shows the comparisons between different methods for the `teddy' and `tsukuba' examples, respectively. It can be clearly seen that the {\em parsimonious labeling} gives better results compared to all the other three methods.  The parameter $w_c$ can be thought of as the trade off between the influence of the pairwise and the high order cliques. Finding the best setting of $w_c$ is very important. The effect of the parameter $w_c$, which is done by changing $\sigma$, is shown in the figure (\ref{fig:stereoTsukubaSigma}). Similarly, the cliques have great impact on the overall result. Large cliques and high value of $w_c$ will result in over smoothing. In order to visualize this, we show the effect of clique size in the figure (\ref{fig:stereoTeddySuperpixels}).

\begin{figure*}[ht!]
        \centering
        \begin{subfigure}[b]{0.18\linewidth}
                \includegraphics[width=\linewidth]{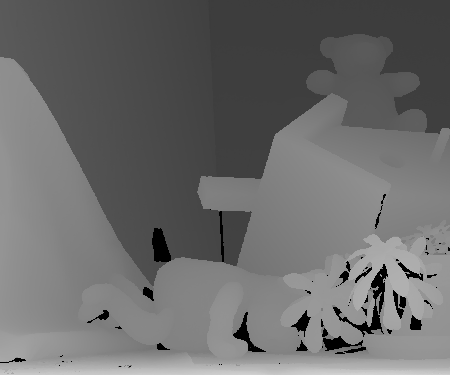}
                \caption{Gnd Truth}      
        \end{subfigure}%
        ~
        \begin{subfigure}[b]{0.18\linewidth}
                \includegraphics[width=\linewidth]{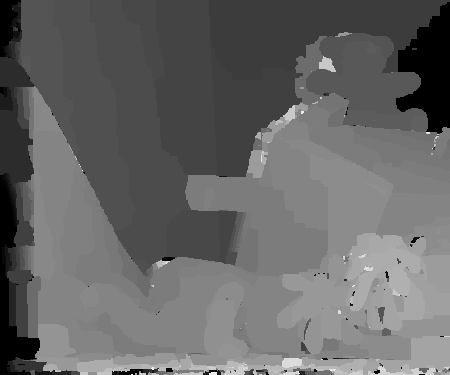}
                \caption{$\alpha-$ exp}      
        \end{subfigure}%
        ~
        \begin{subfigure}[b]{0.18\linewidth}
                \includegraphics[width=\linewidth]{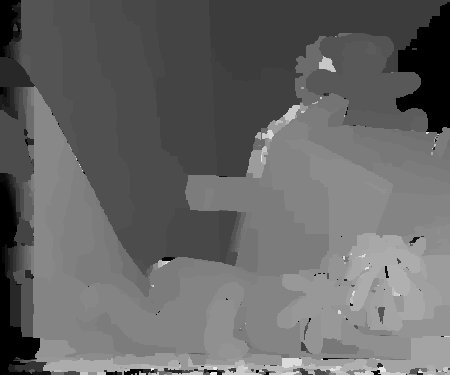}
                \caption{{\sc trws}}      
        \end{subfigure}%
        ~
        \begin{subfigure}[b]{0.18\linewidth}
                \includegraphics[width=\linewidth]{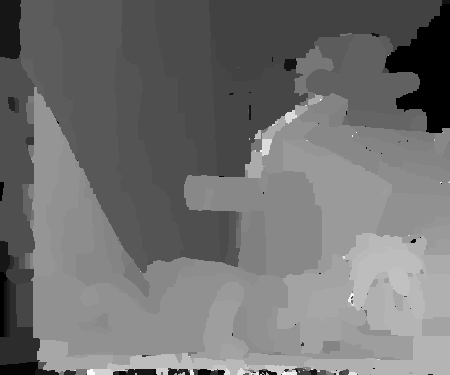}
                \caption{Co-occ}      
        \end{subfigure}%
        ~
        \begin{subfigure}[b]{0.18\linewidth}
                \includegraphics[width=\linewidth]{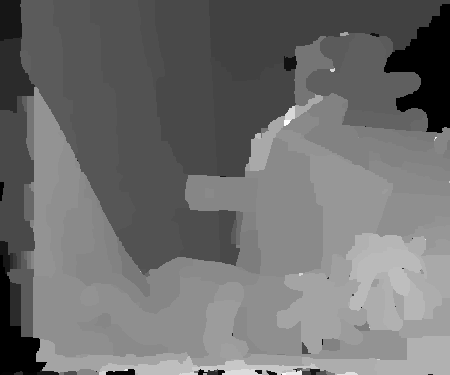}
                \caption{Our Method}      
        \end{subfigure}%

\caption{\label{fig:stereoTeddyComparison}\emph{Comparison of all the methods for the stereo matching of `teddy'. We used the optimal setting of the parameters proposed in the well known Middlebury webpage and \cite{Szeliski_CompStudy_08}. The above results are obtained using $\sigma = 10^2$ for the Co-occ and our method. Clearly, our method gives much smooth results while keeping the underlying shape intact. This is because of the cliques and the corresponding potentials (diversities) used. The diversities enforces smoothness over the cliques while $\sigma$ controls this smoothness in order to avoid over smooth results. }}
\end{figure*}

\begin{figure*}[ht!]
        \centering
        \begin{subfigure}[b]{0.18\linewidth}
                \includegraphics[width=\linewidth]{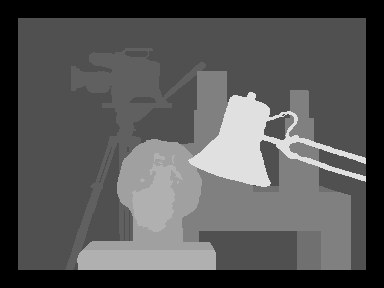}
                \caption{Gnd Truth}      
        \end{subfigure}%
        ~
        \begin{subfigure}[b]{0.18\linewidth}
                \includegraphics[width=\linewidth]{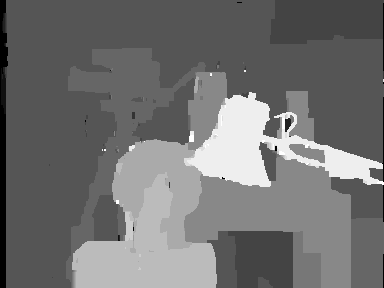}
                \caption{$\alpha-$ exp}      
        \end{subfigure}%
        ~
        \begin{subfigure}[b]{0.18\linewidth}
                \includegraphics[width=\linewidth]{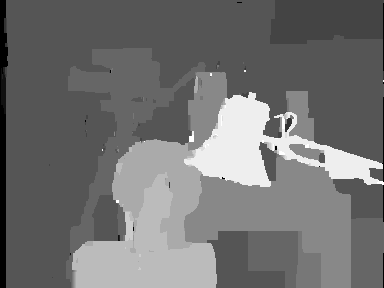}
                \caption{{\sc trws}}      
        \end{subfigure}%
        ~
        \begin{subfigure}[b]{0.18\linewidth}
                \includegraphics[width=\linewidth]{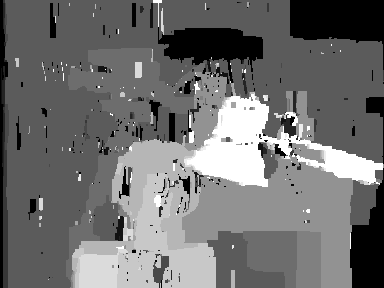}
                \caption{Co-occ}      
        \end{subfigure}%
        ~
        \begin{subfigure}[b]{0.18\linewidth}
                \includegraphics[width=\linewidth]{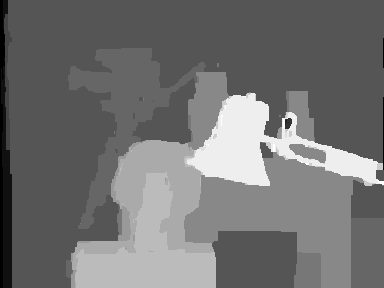}
                \caption{Our Method}      
        \end{subfigure}%

\caption{\label{fig:stereoTsukubaComparison}\emph{Comparison of all the methods for the stereo matching of `tsukuba'. We used the optimal setting of the parameters proposed in the well known Middlebury webpage and \cite{Szeliski_CompStudy_08}. The above results are obtained using $\sigma = 10^2$ for the Co-occ and our method. We can see that the disparity obtained using our method is closest to the ground truth compared to all other methods. In our method, the background is uniform (under the table also), the camera shape is closest to the ground truth camera, and the face disparity is also closest to the ground truth compared to other methods.}}
\end{figure*}

\begin{figure*}[ht!]
        \centering
        \begin{subfigure}[b]{0.18\linewidth}
                \includegraphics[width=\linewidth]{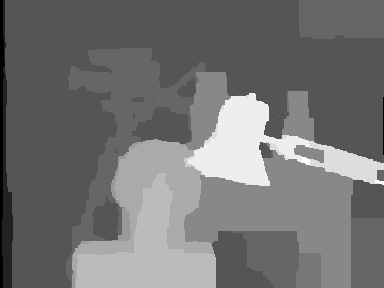}
                \caption{$\sigma = 10^3$}           
        \end{subfigure}%
        ~
        \begin{subfigure}[b]{0.18\linewidth}
                \includegraphics[width=\linewidth]{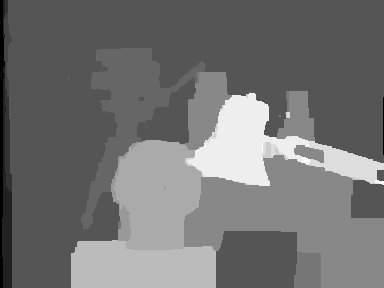}
                \caption{$\sigma = 10^4$}         
        \end{subfigure}%
\caption{\label{fig:stereoTsukubaSigma}\emph{Effect of $\sigma$ in the {\em parsimonious labeling}. All the parameters are same except for the $\sigma$. Note that as we increase the $\sigma$, the $w_c$ increases, which in turn results in over smoothing.}}
\end{figure*}

\begin{figure*}[ht!]
        \centering
        \begin{subfigure}[b]{0.18\linewidth}
                \includegraphics[width=\linewidth]{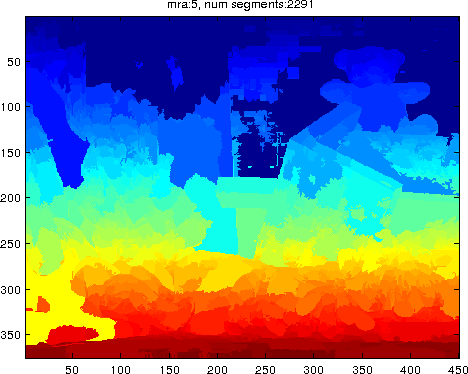}
        \end{subfigure}%
        ~
        \begin{subfigure}[b]{0.18\linewidth}
                \includegraphics[width=\linewidth]{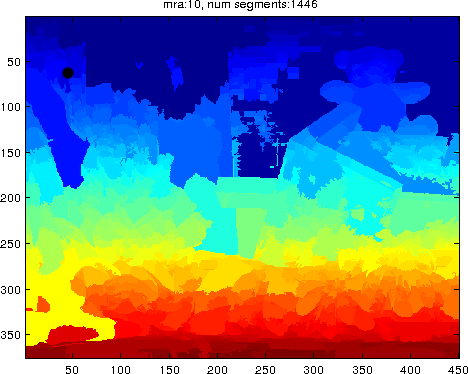}
        \end{subfigure}%
        ~
        \begin{subfigure}[b]{0.18\linewidth}
                \includegraphics[width=\linewidth]{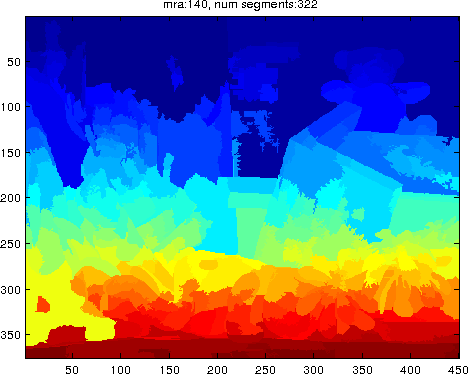}
        \end{subfigure}%
        ~
        \begin{subfigure}[b]{0.18\linewidth}
                \includegraphics[width=\linewidth]{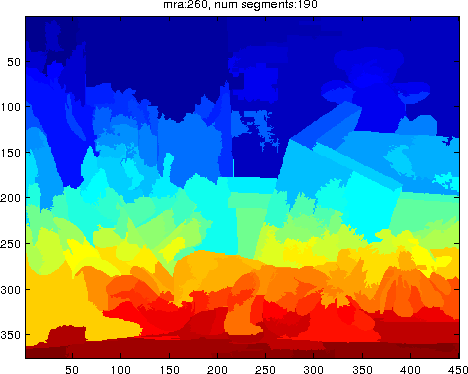}
        \end{subfigure}%
        ~
        \begin{subfigure}[b]{0.18\linewidth}
                \includegraphics[width=\linewidth]{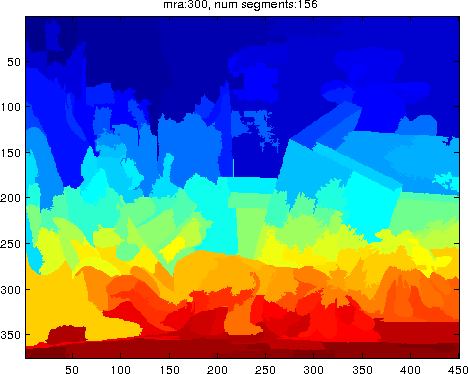}
        \end{subfigure}%
        \\
        \begin{subfigure}[b]{0.18\linewidth}
                \includegraphics[width=\linewidth]{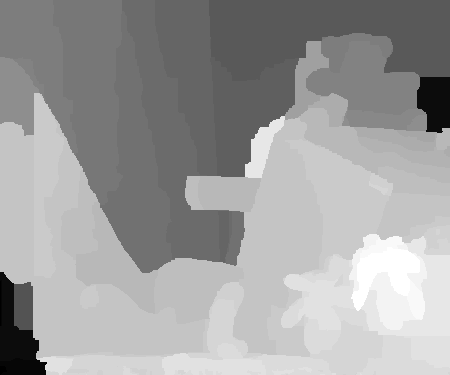}
        \end{subfigure}%
        ~
        \begin{subfigure}[b]{0.18\linewidth}
                \includegraphics[width=\linewidth]{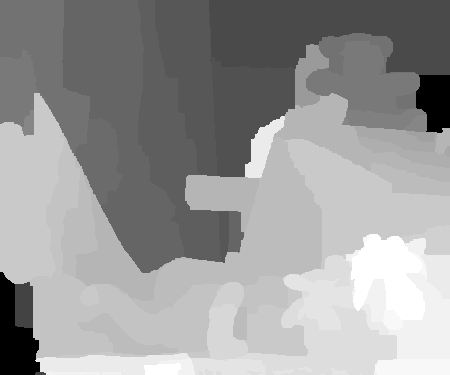}
        \end{subfigure}%
        ~
        \begin{subfigure}[b]{0.18\linewidth}
                \includegraphics[width=\linewidth]{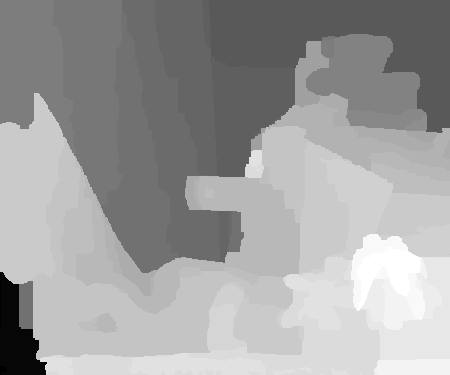}
        \end{subfigure}%
        ~
        \begin{subfigure}[b]{0.18\linewidth}
                \includegraphics[width=\linewidth]{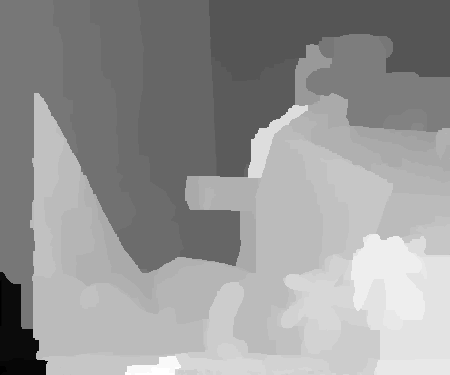}
        \end{subfigure}%
        ~
        \begin{subfigure}[b]{0.18\linewidth}
                \includegraphics[width=\linewidth]{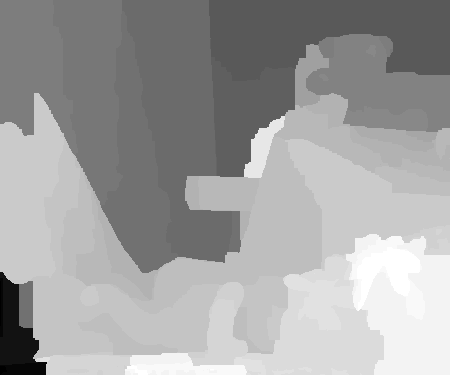}
        \end{subfigure}%

\caption{\label{fig:stereoTeddySuperpixels}\emph{Effect of clique size (superpixels). The top row shows the cliques (superpixels) used and the bottom row shows the stereo matching using these cliques. As we go from left to right, the minimum number of pixels that a superpixel must contain increases. All the other parameters are the same. In order to increase the weight $w_c$, we use high value of $\sigma$, which is $\sigma = 10^5$ in all the above cases.}}
\end{figure*}

\subsection{Image Inpainting and Denoising} 
Please refer to the paper for the description of the image inpainting and the denoising problem. Figures (\ref{fig:inpaintPenguin}) and (\ref{fig:inpaintHouse}) shows the comparisons between the different methods for the `penguin' and the `house' examples, respectively. It can be clearly seen that the {\em parsimonious labeling} gives highly promising results compared to all the other methods.

\begin{figure*}[ht!]
        \centering
        \begin{subfigure}[b]{0.17\linewidth}
                \includegraphics[width=\linewidth]{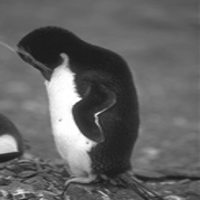}
                \caption{Original}      
        \end{subfigure}%
        ~
        \begin{subfigure}[b]{0.17\linewidth}
                \includegraphics[width=\linewidth]{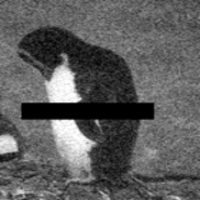}
                \caption{Input}      
        \end{subfigure}%
        ~
        \begin{subfigure}[b]{0.17\linewidth}
                \includegraphics[width=\linewidth]{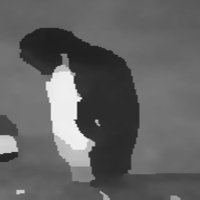}
                \caption{$\alpha-$ exp}      
        \end{subfigure}%
        ~
        \begin{subfigure}[b]{0.17\linewidth}
                \includegraphics[width=\linewidth]{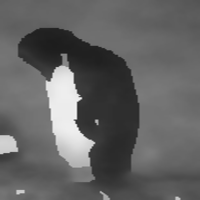}
                \caption{{\sc trws}}      
        \end{subfigure}%
        ~
        \begin{subfigure}[b]{0.17\linewidth}
                \includegraphics[width=\linewidth]{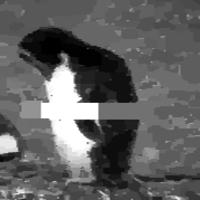}
                \caption{Co-occ}      
        \end{subfigure}%
        ~
        \begin{subfigure}[b]{0.17\linewidth}
                \includegraphics[width=\linewidth]{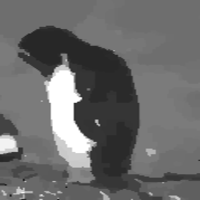}
                \caption{Our}      
        \end{subfigure}%
\caption{\label{fig:inpaintPenguin}\emph{Comparison of all the methods for the image inpainting and denoising problem of the `penguin'. Notice that our method recovers the hand of the penguin very smoothly. In other methods, except Co-oc, the ground is over-smooth while our method recovers the ground quite well compared to others.}}
\end{figure*}

\begin{figure*}[ht!]
        \centering
        \begin{subfigure}[b]{0.17\linewidth}
                \includegraphics[width=\linewidth]{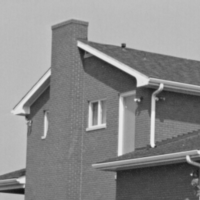}
                \caption{Original}      
        \end{subfigure}%
        ~
        \begin{subfigure}[b]{0.17\linewidth}
                \includegraphics[width=\linewidth]{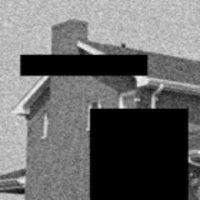}
                \caption{Input}      
        \end{subfigure}%
        ~
        \begin{subfigure}[b]{0.17\linewidth}
                \includegraphics[width=\linewidth]{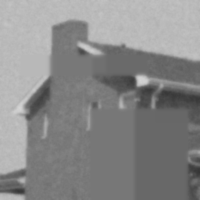}
                \caption{$\alpha-$ exp}      
        \end{subfigure}%
        ~
        \begin{subfigure}[b]{0.17\linewidth}
                \includegraphics[width=\linewidth]{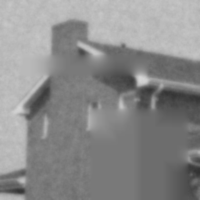}
                \caption{{\sc trws}}      
        \end{subfigure}%
        ~
        \begin{subfigure}[b]{0.17\linewidth}
                \includegraphics[width=\linewidth]{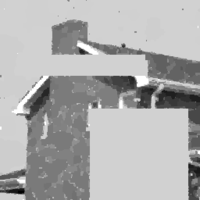}
                \caption{Co-occ}      
        \end{subfigure}%
        ~
        \begin{subfigure}[b]{0.17\linewidth}
                \includegraphics[width=\linewidth]{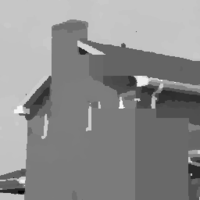}
                \caption{Our}      
        \end{subfigure}%
\caption{\label{fig:inpaintHouse}\emph{Comparison of all the methods for the image inpainting and denoising problem of the `house'.}}
\end{figure*}

\section{Proof of Theorems}
\paragraph{The labeling problem.} As already defined in the paper, consider a random field defined over a set of random variables ${\bf x} = \{x_1, \cdots, x_N\}$ arranged in a predefined lattice $\mathcal{V} = \{1, \cdots, N\}$. Each random variable can take a value from a discrete label set $\mathcal{L} = \{l_1, \cdots, l_H \}$. The energy functional corresponding to a labeling ${\bf x}$ is defined as:
\begin{eqnarray}
E({\bf x}) =   \sum_{i \in \mathcal{V} } \theta_i(x_i)  +\sum_{c \in \mathcal{C}} \theta_c({\bf x}_c)
\label{eq:mrfGeneral1}
\end{eqnarray}
where $\theta_i(x_i)$ is any arbitrary unary potential, and $\theta_c({\bf x}_c)$ is a clique potential for assigning the labels ${\bf x}_c$ to the variables in the clique $c$. 

\paragraph{Notations.} $\Gamma({\bf x}_c)$ denotes the set of unique labels present in the clique ${\bf x}_c$. $\delta(\Gamma({\bf x}_c))$ and $\delta^{dia}(\Gamma({\bf x}_c))$ denotes the {\em diversity} and the {\em diameter diversity} of the unique labels present in the clique ${\bf x}_c$, respectively. $\mathcal{M} = \max_c |{\bf x}_c|$ is the size of the largest maximal-clique and $|\mathcal{L}|$ is the number of labels.

\subsection{Multiplicative Bound of the Hierarchical Move Making Algorithm for the Hierarchical $P^n$ Potts Model - Proof of Theorem-\ref{theorem:rhstMB}}


\begin{proof}
Let ${\bf x}^*$ be the optimal labeling of the given hierarchical $P^n$ Potts model based labeling problem.  Note that any node $p$ in the underlying r-{\sc hst} represents a cluster (subset) of labels. For each node $p$ in the r-{\sc hst} we define following sets using ${\bf x}^*$:
\begin{eqnarray}
\mathcal{L}^p = \{ l_i | l_i \in \mathcal{L}, i \in p \}, \nonumber \\
\mathcal{V}^p = \{ x_i : x_i^* \in \mathcal{L}^p \}, \nonumber \\
\mathcal{I}^p = \{ c: {\bf x}_c \subseteq \mathcal{V}^p \}, \nonumber \\
\mathcal{B}^p = \{ c : {\bf x}_c \cap \mathcal{V}^p \neq \emptyset, {\bf x}_c \nsubseteq  \mathcal{V}^p \}, \nonumber \\
\mathcal{O}^p = \{ c : {\bf x}_c \cap \mathcal{V}^p = \emptyset \}.
\end{eqnarray}
In other words, $\mathcal{L}^p$ is the set of labels in the cluster at $p^{th}$ node, $\mathcal{V}^p$ is the set of nodes whose optimal label lies in the subtree rooted at $p$, $\mathcal{I}^p$ is the set of cliques such that the optimal labeling lies in the subtree rooted at $p$, $\mathcal{B}^p$ is the set of cliques (boundary cliques) such that $\forall {\bf x}_c \in \mathcal{B}^p, \exists \{x_i, x_j\} \in {\bf x}_c : x_i^* \in \mathcal{L}^p, x_j^* \notin \mathcal{L}^p$, and $\mathcal{O}^p$ is the set of outside cliques such that the optimal assignment for all the nodes belongs to the set $\mathcal{L} \setminus \mathcal{L}^p$. Let's define ${\bf x}^p$ as the labeling at node $p$. We prove the following lemma relating ${\bf x}^*$ and ${\bf x}^p$.

\begin{lemma} Let ${\bf x}^p$ be the labeling at node p, ${\bf x}^*$ be the optimal labeling of the given hierarchical $P^n$ Potts model, and $\delta^{dia}(\Gamma({\bf x}_c^p))$ be the diameter diversity based clique potential defined as $\max_{l_i, l_j \in \mathcal{L}^p} d^t(l_i, l_j), \forall p$, where $d^t(.,.)$ is the tree metric defined over the given r-{\sc hst}, then the following bound holds true at any node $p$ of the r-{\sc hst}.
\begin{eqnarray}
\sum_{c \in \mathcal{I}^p} \delta^{dia}(\Gamma({\bf x}_c^p)) \leq   \left( \frac{r}{r-1} \right) \min(\mathcal{M},|\mathcal{L}|)  \sum_{c \in \mathcal{I}^p} \delta^{dia}(\Gamma({\bf x}_c^*))
\end{eqnarray}
\end{lemma}

\begin{proof}
We prove the above lemma by mathematical induction. Clearly, when $p$ is a leaf node, $x_i = p, \forall i \in \mathcal{V}$. For a non-leaf node $p$, we assume that the lemma holds true for the labeling ${\bf x}^q$ of all its children $q$. Given the labeling ${\bf x}^p$ and ${\bf x}^q$, we define a new labeling ${\bf x}^{pq}$ such that
\begin{eqnarray}
\label{eq:alphaExpRhst}
     {\bf x}^{pq}
      =
      \Big\{
      \begin{array}{c}
      x_i^q \textit{   if  } x_i^* \in \mathcal{L}^q \\
      x_i^p \textit{ otherwise.}
       \end{array}
\end{eqnarray}
Note that ${\bf x}^{pq}$ lies within one $\alpha$-expansion iteration away from ${\bf x}^{p}$. Since ${\bf x}^p$ is the local minima, we can say that
\begin{eqnarray}
E({\bf x}^{p} | \mathcal{I}^p) + E({\bf x}^{p} | \mathcal{B}^p) + E({\bf x}^{p} | \mathcal{O}^p) &\leq& E({\bf x}^{pq} | \mathcal{I}^{pq}) + E({\bf x}^{pq} | \mathcal{B}^{pq}) + E({\bf x}^{pq} | \mathcal{O}^{pq}) \nonumber \\
\label{eq:rhstIneq1}
E({\bf x}^{p} | \mathcal{I}^p) + E({\bf x}^{p} | \mathcal{B}^p) &\leq& E({\bf x}^{pq} | \mathcal{I}^{pq}) + E({\bf x}^{pq} | \mathcal{B}^{pq})  \\
\sum_{c \in \mathcal{I}^p} \delta^{dia}(\Gamma({\bf x}_c^p)) + \sum_{c \in \mathcal{B}^p} \delta(\Gamma({\bf x}_c^p)) &\leq& \sum_{c \in \mathcal{I}^{pq}} \delta^{dia}(\Gamma({\bf x}_c^{pq})) + \sum_{c \in \mathcal{B}^{pq}} \delta^{dia}(\Gamma({\bf x}_c^{pq})) \\
\sum_{c \in \mathcal{I}^q} \delta^{dia}(\Gamma({\bf x}_c^p)) + \sum_{c \in \mathcal{B}^q} \delta^{dia}(\Gamma({\bf x}_c^p)) &\leq& \sum_{c \in \mathcal{I}^{q}} \delta^{dia}(\Gamma({\bf x}_c^{pq})) + \sum_{c \in \mathcal{B}^{q}} \delta^{dia}(\Gamma({\bf x}_c^{pq})) 
\end{eqnarray}
 Using the mathematical induction we can write 
\begin{eqnarray}
\label{ineq:rhstIneqAllq}
\sum_{c \in \mathcal{I}^q} \delta^{dia}(\Gamma({\bf x}_c^p)) + \sum_{c \in \mathcal{B}^q} \delta^{dia}(\Gamma({\bf x}_c^p)) &\leq& \min(\mathcal{M}, |\mathcal{L}|)\left( \frac{r}{r-1} \right) \sum_{c \in \mathcal{I}^{q}} \delta^{dia}(\Gamma({\bf x}_c^*)) + \sum_{c \in \mathcal{B}^{q}} \delta^{dia}(\Gamma({\bf x}_c^{pq})) \nonumber \\
\end{eqnarray}
Now consider a clique $c \in \mathcal{B}^q$. Let $e^p$ be the length of edges from node $p$ to its children $q$. Since $c \in \mathcal{B}^q$, there must exist atleast two nodes $x_i$ and $x_j$ in ${\bf x}_c$ such that $x_i^* \in \mathcal{L}^q$ and $x_j^* \notin \mathcal{L}^q$, therefore, by construction of r-{\sc hst}
\begin{equation}
\delta^{dia}(\Gamma({\bf x}_c^*)) \geq 2 e^p
\end{equation}

Furthermore, by the construction of ${\bf x}^{pq}$, $\mathcal{L}^{pq} \subseteq \mathcal{L}^p$, therefore, in worst case (leaf nodes), we can write
\begin{eqnarray}
\delta^{dia}(\Gamma({\bf x}_c^{pq})) = max_{l_i, l_j \in \mathcal{L}^{pq}} d^t(l_i, l_j) &\leq& 2e^p \left( 1 + \frac{1}{r} + \frac{1}{r^2} + \cdots \right) \nonumber \\
&=& 2e^p \left( \frac{r}{r-1} \right) \nonumber \\
\label{ineq:pqVsOptimal}
&\leq& \delta^{dia}(\Gamma({\bf x}_c^*)) \left( \frac{r}{r-1} \right)
\end{eqnarray} 
From inequalities (\ref{ineq:rhstIneqAllq}) and (\ref{ineq:pqVsOptimal})
\begin{eqnarray}
\label{ineq:qFinal}
\sum_{c \in \mathcal{I}^q} \delta^{dia}(\Gamma({\bf x}_c^p)) + \sum_{c \in \mathcal{B}^q} \delta^{dia}(\Gamma({\bf x}_c^p)) \leq && \min(\mathcal{M}, |\mathcal{L}|) \left( \frac{r}{r-1} \right) \sum_{c \in \mathcal{I}^{q}} \delta^{dia}(\Gamma({\bf x}_c^*)) + \nonumber \\
&& \left( \frac{r}{r-1} \right) \sum_{c \in \mathcal{B}^{q}} \delta^{dia}(\Gamma({\bf x}_c^*))  \nonumber \\
\end{eqnarray}
In order to get the bound over the total energy we sum over all the children $q$ of $p$, denoted as $\eta(p)$. Therefore, summing the inequality (\ref{ineq:qFinal}) over $\eta(p)$ we get

\begin{eqnarray}
\label{ineq:qFinalAll}
\sum_{q \in \eta(p)} \sum_{c \in \mathcal{I}^q} \delta^{dia}(\Gamma({\bf x}_c^p)) +  \sum_{q \in \eta(p)} \sum_{c \in \mathcal{B}^q} \delta^{dia}(\Gamma({\bf x}_c^p)) &\leq&  
\min(\mathcal{M}, |\mathcal{L}|)\left( \frac{r}{r-1} \right) \sum_{q \in \eta(p)} \sum_{c \in \mathcal{I}^{q}} \delta^{dia}(\Gamma({\bf x}_c^*)) \nonumber \\
&+& \left( \frac{r}{r-1} \right) \sum_{q \in \eta(p)} \sum_{c \in \mathcal{B}^{q}} \delta^{dia}(\Gamma({\bf x}_c^*))  
\end{eqnarray}

The LHS of the above inequality can be written as
\begin{eqnarray}
\label{ineq:rhstLHSfinal}
\sum_{q \in \eta(p)} \sum_{c \in \mathcal{I}^q} \delta^{dia}(\Gamma({\bf x}_c^p)) + \sum_{q \in \eta(p)} \sum_{c \in \mathcal{B}^q} \delta^{dia}(\Gamma({\bf x}_c^p)) &\geq& \sum_{c \in \cup_{q \in \eta(p)} \mathcal{I}^q } \delta^{dia}(\Gamma({\bf x}_c^p)) + \sum_{c \in \cup_{q \in \eta(p)} \mathcal{B}^q } \delta^{dia}(\Gamma({\bf x}_c^p)) \nonumber \\
&=& \sum_{c \in \mathcal{I}^p} \delta^{dia}(\Gamma({\bf x}_c^p))
\end{eqnarray}
The above inequality and equality is due to the fact that $\cap_{q \in \eta(p) } \mathcal{I}^{q} = \emptyset$, $\cap_{q \in \eta(p) } \mathcal{B}^{q}$ is not necessarily an empty set, $\delta^{dia}(\Gamma({\bf x}_c)) \geq 0$, and $\mathcal{I}^p = \{ \cup_{q \in \eta(p)} \mathcal{I}^q \} \cup \{ \cup_{q \in \eta(p)}\mathcal{B}^q \}$. Now let us have a look into the second term of the RHS of the inequality (\ref{ineq:qFinalAll})

\begin{eqnarray}
\label{ineq:rhstML1}
\sum_{q \in \eta(p)} \sum_{c \in \mathcal{B}^q} \delta^{dia}(\Gamma({\bf x}_c^*)) &\leq& \sum_{c \in \cup_{q \in \eta(p)} \mathcal{B}^{q}} \min(|\eta(p)|, |{\bf x}_c|) \delta^{dia}(\Gamma({\bf x}_c^*)) \\
&\leq& \min \left( \max_{p \in \eta(p)} |\eta(q)|, \max_c |{\bf x}_c| \right) \sum_{c \in \cup_{q \in \eta(p)} \mathcal{B}^{q}}  \delta^{dia}(\Gamma({\bf x}_c^*)) \nonumber \\
\label{ineq:rhstML2}
&=& \min (\mathcal{L}, |\mathcal{M}|) \sum_{c \in \cup_{q \in \eta(p)} \mathcal{B}^{q}}  \delta^{dia}(\Gamma({\bf x}_c^*)) 
\end{eqnarray}
The inequality (\ref{ineq:rhstML1}) is due to the fact that $\cup_{q \in \eta(p)} \mathcal{B}^{q}$ can not count a clique more than $\min(|\eta(p)|, |{\bf x}_c|)$ times. Therefore, using the inequality (\ref{ineq:rhstML2}) in the RHS of the inequality (\ref{ineq:rhstLHSfinal}) we get

\begin{eqnarray}
&& \min(\mathcal{M}, |\mathcal{L}|)\left( \frac{r}{r-1} \right) \sum_{q \in \eta(p)} \sum_{c \in \mathcal{I}^{q}} \delta^{dia}(\Gamma({\bf x}_c^*)) + \left( \frac{r}{r-1} \right) \sum_{q \in \eta(p)}  \sum_{c \in \mathcal{B}^{q}} \delta^{dia}(\Gamma({\bf x}_c^*))  \nonumber \\
&\leq& \min(\mathcal{M}, |\mathcal{L}|)\left( \frac{r}{r-1} \right) \left( \sum_{c \in \cup_{q \in \eta(p)} \mathcal{I}^q} \delta^{dia}(\Gamma({\bf x}_c^*)) + \sum_{c \in \cup_{q \in \eta(p)} \mathcal{B}^q} \delta^{dia}(\Gamma({\bf x}_c^*)) \right) \nonumber \\
\label{ineq:rhstRHSfinal}
&=&\min(\mathcal{M}, |\mathcal{L}|)\left( \frac{r}{r-1} \right) \sum_{c \in \mathcal{I}^p} \delta^{dia}(\Gamma({\bf x}_c^*))
\end{eqnarray}

Finally, using inequalities (\ref{ineq:qFinalAll}), (\ref{ineq:rhstLHSfinal}) and (\ref{ineq:rhstRHSfinal}) we get
\begin{eqnarray}
\sum_{c \in \mathcal{I}^p} \delta^{dia}(\Gamma({\bf x}_c^p)) \leq \min(\mathcal{M}, |\mathcal{L}|) \left( \frac{r}{r-1} \right) \sum_{c \in \mathcal{I}^p} \delta^{dia}(\Gamma({\bf x}_c^*))
\end{eqnarray}

\end{proof}
Applying the above lemma to the root node proves the theorem.
\end{proof}

%

\subsection{Multiplicative Bound of the Algorithm-\ref{algo:parsimonyMoveMaking} for the Parsimonious Labeling - Proof of Theorem-\ref{theo:diversityMB}}
\begin{proof}
Let us say that $d(.,.)$ is the induced metric of the given diversity $(\delta, \mathcal{L})$ and $\delta^{dia}$ be it's diameter diversity. We first approximate $d(.,.)$ as a mixture of r-{\sc hst} metrics $d^t(.,.)$. Using Theorem-\ref{theo:treeMetric} we get the following relationship
\begin{eqnarray}
\label{eq:metricTotreeMetric}
d(.,.) \leq O(\log |\mathcal{L}|) d^t(.,.)
\end{eqnarray}

For a given clique ${\bf x}_c$, using Proposition-\ref{prop:divInequality}, we get the following relationship
\begin{eqnarray}
\label{eq:diaInequality}
\delta^{dia} (\Gamma({\bf x}_c)) \leq \delta(\Gamma({\bf x}_c)) \leq (|\Gamma({\bf x}_c)|-1)\delta^{dia} (\Gamma({\bf x}_c))
\end{eqnarray}
Therefore, using equations (\ref{eq:diaInequality}) and (\ref{eq:metricTotreeMetric}), we get the following inequality
\begin{eqnarray}
\label{eq:diaIneuqlityFinal}
\delta^{dia} (\Gamma({\bf x}_c)) \leq \delta(\Gamma({\bf x}_c)) &\leq& (|\Gamma({\bf x}_c)|-1)\delta^{dia} (\Gamma({\bf x}_c)) \nonumber \\
&\leq& O(\log |\Gamma({\bf x}_c)|) (|\Gamma({\bf x}_c)|-1)\delta^{dia}_t (\Gamma({\bf x}_c))
\end{eqnarray}
where, $\delta^{dia}_t (\Gamma({\bf x}_c))$ is the diameter diversity defined over the tree metric $d^t(.,.)$ which is obtained using the randomized algorithm \cite{Fakcharoenphol04TightBound} on the induced metric $d(.,.)$.

Hence, combing the inequality (\ref{eq:diaIneuqlityFinal}) and the previously proved Theorem-\ref{theorem:rhstMB} proves the Theorem-\ref{theo:diversityMB}.

Notice that, in case our diversity in itself is a diameter diversity, we don't need the inequality (\ref{eq:diaInequality}), therefore, the multiplicative bound reduces to $\left( \frac{r}{r-1} \right) (\log |\mathcal{L}|) \min(\mathcal{M},|\mathcal{L}|)$.
\end{proof}



\begin{theorem}
\label{theo:treeMetric}
Given any distance metric function $d(.,.)$ defined over a set of labels $\mathcal{L}$, the randomized algorithm given in \cite{Fakcharoenphol04TightBound} produces a mixture of r{-\sc hst} tree metrics $d^t(.,.)$ such that $d(.,.) \leq O(\log |\mathcal{L}|) d^t(.,.)$.
\\ Proof: Please see the reference \cite{Fakcharoenphol04TightBound}.
\end{theorem}

\begin{proposition}
\label{prop:divInequality}
Let $(\mathcal{L}, \delta)$ be a diversity with induced metric space $(\mathcal{L}, d)$, then the following inequality holds $\forall \Gamma \subseteq \mathcal{L}$.
\begin{eqnarray}
\delta^{dia} (\Gamma) \leq \delta(\Gamma) \leq (|\Gamma|-1)\delta^{dia} (\Gamma)
\label{eq:divInequality}
\end{eqnarray}
Proof: Please see the reference \cite{Bryant13Diversities}.
\end{proposition}


\twocolumn

{\small
\bibliographystyle{ieee}
\bibliography{dokaniaBibliography}

\begin{thebibliography}{10}\itemsep=-1pt

\bibitem{Bartal98rhst}
Y.~Bartal.
\newblock On approximating arbitrary metrics by tree metrics.
\newblock In {\em STOC}, 1998.

\bibitem{Boykov01GraphCut}
Y.~Boykov, O.~Veksler, and R.~Zabih.
\newblock Fast approximate energy minimization via graph cuts.
\newblock In {\em PAMI}, 2001.

\bibitem{Bryant12Hyperconvexity}
D.~Bryant and P.~F. Tupper.
\newblock Hyperconvexity and tight-span theory for diversities.
\newblock In {\em Advances in Mathematics}, 2012.

\bibitem{Bryant13Diversities}
D.~Bryant and P.~F. Tupper.
\newblock Diversities and the geometry of hypergraphs.
\newblock In {\em Discrete Mathematics and Theoretical Computer Science}, 2014.

\bibitem{MeanShift_PAMI_02}
D.~Comaniciu and P.~Meer.
\newblock Mean shift: A robust approach toward feature space analysis.
\newblock In {\em PAMI}, 2002.

\bibitem{Delong_HierCost_2012}
A.~Delong, L.~Gorelick, O.~Veksler, and Y.~Boykov.
\newblock Minimizing energies with hierarchical costs.
\newblock In {\em IJCV}, 2012.

\bibitem{Delong_LabelCost_2010}
A.~Delong, A.~Osokin, H.~Isack, and Y.~Boykov.
\newblock Fast approximate energy minimization with label costs.
\newblock In {\em CVPR}, 2010.

\bibitem{Zehiry_Curvature_10}
N.~El-Zehiry and L.~Grady.
\newblock Fast global optimization of curvature.
\newblock In {\em CVPR}, 2010.

\bibitem{Fakcharoenphol04TightBound}
J.~Fakcharoenphol, S.~Rao, and K.~Talwar.
\newblock A tight bound on approximating arbitrary metrics by tree metrics.
\newblock In {\em STOC}, 2003.

\bibitem{Fitzgibbon_Rendering_03}
A.~Fitzgibbon, Y.~Wexler, and A.~Zisserman.
\newblock Image-based rendering using image-based priors.
\newblock In {\em ICCV}, 2003.

\bibitem{Gould_AlphabetSoup_09CVPR}
S.~Gould, F.~Amat, and D.~Koller.
\newblock Alphabet soup: A framework for approximate energy minimization.
\newblock In {\em CVPR}, 2009.

\bibitem{Kleinberg_Metric_1999}
J.~Kleinberg and E.~Tardos.
\newblock Approximation algorithms for classification problems with pairwise
  relationships: Metric labeling and markov random fields.
\newblock In {\em FOCS}, 1999.

\bibitem{Kohli07_P3Beyond_CVPR}
P.~Kohli, M.~Kumar, and P.~Torr.
\newblock P3 \& beyond: Solving energies with higher order cliques.
\newblock In {\em CVPR}, 2007.

\bibitem{Kohli_RobustPn_2008}
P.~Kohli, L.~Ladicky, and P.~Torr.
\newblock Robust higher order potentials for enforcing label consistency.
\newblock In {\em CVPR}, 2008.

\bibitem{Kolmogorov_TRWS_06PAMI}
V.~Kolmogorov.
\newblock Convergent tree-reweighted message passing for energy minimization.
\newblock In {\em PAMI}, 2006.

\bibitem{Pawan14_Rounding_NIPS}
M.~P. Kumar.
\newblock Rounding-based moves for metric labeling.
\newblock In {\em NIPS}, 2014.

\bibitem{Pawan_SemiMetric_09UAI}
M.~P. Kumar and D.~Koller.
\newblock {MAP} estimation of semi-metric {MRF}s via hierarchical graph cuts.
\newblock In {\em UAI}, 2009.

\bibitem{Ladicky_COOC_2010}
L.~Ladicky, C.~Russell, P.~Kohli, and P.~Torr.
\newblock Graph cut based inference with co-occurrence statistics.
\newblock In {\em ECCV}, 2010.

\bibitem{Lan_BPHighOrder_06}
X.~Lan, S.~Roth, D.~Huttenlocher, and M.~Black.
\newblock Efficient belief propagation with learned higher-order markov random
  fields.
\newblock In {\em ECCV}, 2006.

\bibitem{Potts_1952}
R.~B. Potts.
\newblock Some generalized order-disorder transformations.
\newblock In {\em Cambridge Philosophical Society}, 1952.

\bibitem{Middlebury_Stereo_02}
D.~Scharstein and R.~Szeliski.
\newblock A taxonomy and evaluation of dense two-frame stereo correspondence
  algorithms.
\newblock In {\em IJCV}, 2002.

\bibitem{Szeliski_CompStudy_08}
R.~Szeliski, R.~Zabih, D.~Scharstein, O.~Veksler, V.~Kolmogorov, A.~Agarwala,
  M.~Tappen, and C.~Rother.
\newblock A comparative study of energy minimization methods for markov random
  fields with smoothness-based priors.
\newblock In {\em PAMI}, 2008.

\bibitem{Tappen_Stereo_03}
M.~Tappen and W.~Freeman.
\newblock Comparison of graph cuts with belief propagation for stereo, using
  identical mrf parameters.
\newblock In {\em ICCV}, 2003.

\bibitem{Tarlow_HOP_10}
D.~Tarlow, I.~Givoni, and R.~Zemel.
\newblock Hop-map: Efficient message passing with high order potentials.
\newblock In {\em AISTATS}, 2010.

\bibitem{Veksler_PhDthesis_1999}
O.~Veksler.
\newblock Efficient graph-based energy minimization methods in computer vision,
  1999.

\bibitem{Vicente_Joint_09}
S.~Vicente, V.~Kolmogorov, and C.~Rother.
\newblock Joint optimization of segmentation and appearance models.
\newblock In {\em ICCV}, 2009.

\bibitem{Woodford_Stereo_08}
O.~Woodford, P.~Torr, I.~Reid, and A.~Fitzgibbon.
\newblock Global stereo reconstruction under second order smoothness priors.
\newblock In {\em CVPR}, 2008.

\end{thebibliography}
}
\end{document}